\documentclass[letterpaper, 10 pt, conference]{ieeeconf}  

\usepackage{subcaption}
\usepackage{caption}
\usepackage{mathptmx} 
\usepackage{times} 
\usepackage{amsmath} 

\usepackage{amssymb}  
\usepackage{mathrsfs}                                  
\usepackage{graphics}
\usepackage{graphicx} 
\usepackage{float}
\usepackage{soul}
\usepackage{comment}
\usepackage{algorithmicx}
\usepackage{algpseudocode}
\usepackage[ruled,vlined]{algorithm2e}
\usepackage[english]{babel}
\usepackage{tikz}
\usepackage{ftnxtra}
\usepackage{stackengine}
\usepackage{color}
\DeclareMathAlphabet{\mathcal}{OMS}{cmsy}{m}{n}
\newtheorem{remark}{Remark}
\newtheorem{theorem}{Theorem}

\usepackage{hyperref}
\usepackage[autostyle]{csquotes}  

\newif\ifdraft
\draftfalse

\IEEEoverridecommandlockouts                              

\title{\LARGE \bf 
	Parallax Bundle Adjustment on Manifold with Convexified Initialization 	
}

\author{Liyang Liu$^1$, Teng Zhang$^1$, Yi Liu$^2$, Brenton Leighton$^1$,   Liang Zhao$^1$,  Shoudong Huang$^1$ and Gamini Dissanayake$^1$
\thanks{$^1$Liyang Liu, Teng Zhang, Brenton Leighton, Liang Zhao, Shoudong Huang and Gamini Dissanayake  are with the Center for Autonomous Systems (CAS), University of Technology Sydney, Ultimo, NSW 2007, Australia.
	 {\tt \small  \{Liyang.Liu, Teng.Zhang, 
	 Brenton.Leighton, Liang.Zhao, Shoudong.Huang, Gamini.Dissanayake\}@uts.edu.au}
  }
\thanks{$^2$ Yi Liu is with School of Automation, Huazhong University of Science and Technology, Wuhan 430074, China.
{\tt \small  \{skyridermike\}@hust.edu.cn}
} 
}

\begin{document}

\maketitle
\thispagestyle{empty}
\pagestyle{empty}

\begin{abstract}

Bundle adjustment (BA) with parallax angle based feature parameterization has been shown to have superior performance over BA using inverse depth or XYZ feature forms.  In this paper, we propose an improved version of the parallax BA algorithm (PMBA) by extending it to the manifold domain along with observation-ray
based objective function. With this modification, the problem formulation faithfully mimics the
projective nature in a camera's image formation, BA is able to achieve better convergence, accuracy and robustness. This is particularly useful in handling diverse outdoor environments and collinear motion modes. Capitalizing on these properties, we further propose a pose-graph simplification
to PMBA, with significant dimensionality reduction. This pose-graph model is convex in nature, easy to solve and its solution can serve as a good initial guess to the original BA problem which is intrinsically non-convex. We provide theoretical proof that our global initialization strategy can guarantee a near-optimal
solution. Using a series of experiments involving diverse environmental conditions and motions, 
we demonstrate PMBA's superior convergence performance in comparison to other BA methods. We also show that, 
without incremental initialization or via third-party information, our global initialization process helps to 
bootstrap the full BA successfully in various scenarios, sequential or out-of-order, including some datasets 
from the ``Bundle Adjustment in the Large'' database.

\end{abstract}

\section{Introduction}
Structure from Motion (SfM) / visual SLAM estimates 3D scene structures and camera poses simultaneously from 2D images. Bundle adjustment is the gold standard method of SfM, in that it finds optimal pose and map in the least squares sense to best explain the data. 
Solving such a non-linear least squares problem typically requires iterative Newton methodology: start with an initial guess, repetitively add increments by solving a normal equation until convergence. 
\begin{table}[h]
\centering
\caption{\small Three types of Newton-based methods} 
\label{table::opt-method}
\begin{tabular}{|c|c|c|}
\hline
GN & LM & DL \\
\hline
$\triangle \mathbf{x} = \mathbf{H}^{-1}\mathbf{e}(\mathbf{x})$ 
& $\triangle \mathbf{x} = (\mathbf{H}+\lambda \mathbf{I})^{-1}\mathbf{e}(\mathbf{x})$ 
& $\triangle \mathbf{x} = ( \lambda_1\mathbf{H}^{-1}+\lambda_2\mathbf{I})\mathbf{e}(\mathbf{x})$ \\
\hline
 \end{tabular}
\end{table}

\begin{figure}[t]
	\centering
	\includegraphics[width=0.95\linewidth,height=4cm]{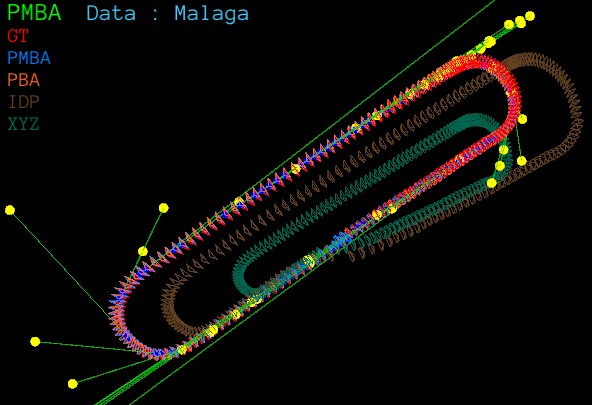}
	\caption{\footnotesize Compare BA for ``Malaga dataset'': existence of collinear features (yellow dots) cause IDP (brown) and XYZ (green) to differ significantly from Ground Truth (red); PMBA (blue) and PBA \cite{PBAliang} (orange) do not show this issue, with PMBA having the fastest convergence rate, see Fig. \ref{fig::converge}(\emph{a}).}
	\label{fig::malaga-collinear}
\end{figure}

As shown in Table \ref{table::opt-method}, this approach comes in three forms: original Gauss-Newton (GN) when the equation is easy to solve (the Hessian matrix $\mathbf{H}$ has a small condition number), Levenberg Macquardt (LM) as a damped GN if Hessian is near singular, and DogLeg (DL) as a combination of GN and the steepest descent method for fast convergence. LM is a favourite of the BA community for its safe handling despite its slowness. GN and DL are both considered risky due to the large step size and are often avoided.\\
\\
\textbf{Problematic features}\\
In many modern BA systems \cite{ceres-solver}\cite{g2o}\cite{GTSAM}, a 3D feature point is parameterized as Euclidean coordinates (XYZ) or inverse depth (IDP). A well-known problem for these representations is that when far away features exist or when camera poses observing a feature are collinear with the feature, the Hessian becomes ill-conditioned. A small change in error function leads to a large jump in the state variable, significantly affecting BA's robustness, efficiency and accuracy. See Fig. \ref{fig::malaga-collinear} and Fig. \ref{fig::converge}(\emph{a}) for illustration of failure in conventional BA.

To deal with this problem, several remedies are commonly adopted. The fundamental principle is separate treatment for problematic features and good
ones. ORB-SLAM \cite{murTRO2015} uses a prudent feature
selection strategy where features with in-sufficient parallax angles are  discarded. A hybrid method was proposed in \cite{farPoints}, that first estimates camera orientations with remote features then optimises with poses and near features. The vision smart factor proposed in \cite{smartFactor} (implemented in GTSAM \cite{GTSAM}) shares the same approach of \cite{farPoints}. It avoids degenerate cases by using a flexible-size error function.
Recently \cite{FeatureConfidence} proposed a solution in which less weighting is given to the error
 terms for “problematic” features.

Compared to the afore-mentioned methods, our proposed algorithm PMBA treats
the problem with a totally different viewpoint. We argue that \textbf{ the root cause for ill-conditioned cases is that feature uncertainty for conventional BAs is NOT uniformly bounded}. In our previous work \cite{PBAliangICRA}\cite{PBAliang}, we used three angles (elevation, azimuth and parallax) to define structure of a feature without involving depth. \cite{PBAliang} demonstrated that this parameterization is closer to the measurement space of projective geometry, parallax-based BA (we call it PBA in this paper) is
more robust and efficient compared to BA's in XYZ or IDP form. We will present our improved manifold version --  \textbf{PMBA that faithfully complies with projective geometry in computer vision. This results in a non-singular Hessian and a bounded error function that is suitable for faster implementation}.\\
\\
\textbf{Initialization methods}\\
BA due to its highly non-convex nature, requires good initial estimate to converge to global minimum. The common initialization methods include
incremental or global. In incremental strategy, with a simple start, many mid-level BAs are performed on each new pose insertion. Incremental
strategy draws the criticism that it is slow and leads to drifting for long sequences of data. Example systems are VisualSFM \cite{wu2011visualsfm},  Bundler \cite{bundler} and ORB-SLAM \cite{murTRO2015}. The alternative is global initialization where all camera poses are initialised simultaneously. Global SfM thus bootstrapped shows higher efficiency and accuracy. This strategy exposes many research challenges, and has been studied intensively in \cite{LinearGlobalTranslation}\cite{Njiang}\cite{tang2017globalslam}\cite{TanGlobal}.

This paper builds on the previous PBA algorithm \cite{PBAliangICRA}\cite{PBAliang} and makes the following improvements:
(1) recalls the conventional BA methods and analyzes its limitations (Section \ref{Section::Background}); (2) an improved PBA on manifold formulation that is able to fully avoid ``problematic feature'' induced ill-conditioned cases (Section \ref{Section::PMBA}) ; (3) a simple but effective global initialization method using convexified pose-graph model that is compatible with PMBA, which can guarantee a near-optimal solution (Section \ref{Section::Init}); 
(4) to demonstrate the two improvements, we provide both theoretical proof and experimental results from a series of large-scale datasets, sequential or out-of-order (Section \ref{Section::Experiment}).



\textbf{Notations:} 
\begin{itemize}
\item ${S}(\mathbf{x} )$ is a skew symmetric matrix from vector $\mathbf{x}\in\mathbb{R}^3$, equivalent to cross-product operator, ${S}(\mathbf{x})\mathbf{y}=\mathbf{x}\times \mathbf{y}$
 \item The term $\mathbf{T}_i = (\mathbf{R}_i, \mathbf{p}_i) \in \mathbb{SE}(3)$ represents the camera pose at time-step $i$. 
\item Subscript ${}^{(l)}$ indicates frame is local.
\item Decoration $\,\breve{}\,$ indicates vector is normalized:  $\breve{\mathbf{N}}_{j,i}=\frac{\mathbf{N}_{j,i}}{\|\mathbf{N}_{j,i}\|}$.
\end{itemize}

\section{Background Knowledge}
\label{Section::Background}
In this section, we first recall the monocular SLAM problem and conventional BA. We then  analyze the potential problems in this formulation.


The visual SLAM problem estimates camera poses $\mathbf{T} = \{ (\mathbf{R}_i, \mathbf{p}_i) \}_{i=1,\cdots,M}$ and feature positions $\mathbf{f}=\{ \mathbf{f}_j \in \mathbb{R}^3 \}_{j=1,\cdots,N}$   from a set of images $\{ I_i \}$. When the feature $j$ is observed from the pose $\mathbf{T}_i$,  
the monocular sensor intercepts the light ray $\mathbf{N}_{j,i}$
that passes through its
centre to the feature point in the form of image pixel $\mathbf{u}_{m_{j,i}}$, as shown in Fig. \ref{fig::DirErr}(\emph{a}). Table \ref{table::ray-forms} lists different expressions the observation ray can have.

\begin{figure}[htbp]
	\centering
	\begin{tabular}{cc}
	\includegraphics[width=0.53\linewidth,height=5cm]{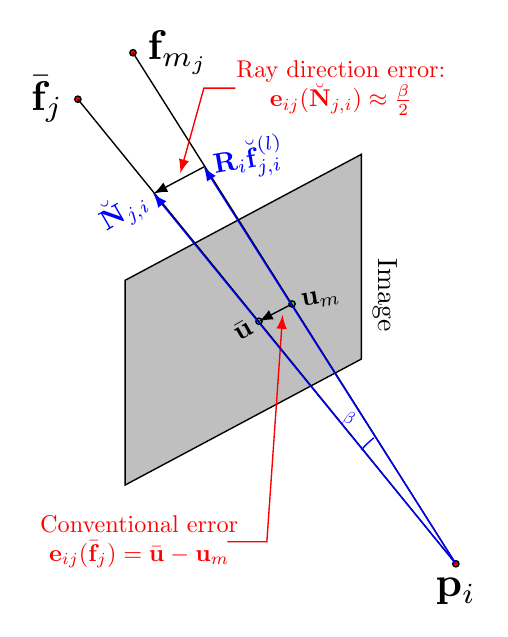} &
	\includegraphics[width=0.42\linewidth,height=5cm]{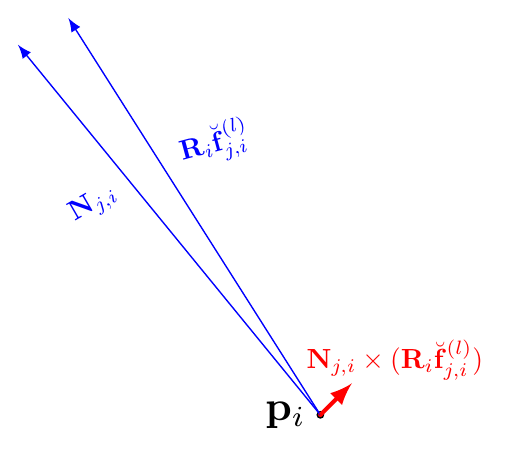} \\
{\small 
\begin{tabular}{@{}l@{}}(\emph{a}) 
Camera measurement formation\\
and BA error functions
\end{tabular}}
 &
\footnotesize{\begin{tabular}{@{}l@{}}(\emph{b}) PMBA reformatted as a QPLC\\ 
problem: minimize cross product \end{tabular}
} \\
\end{tabular}
	\caption{Projective vision and error functions in PMBA}
	\label{fig::DirErr}
\end{figure}

\begin{table}[h]
\centering
\caption{\small Various forms of observation ray in this paper} 
\label{table::ray-forms}
\begin{tabular}{|c|c|c|c|}
\hline
\footnotesize{Global ray} & 		
\footnotesize{ \begin{tabular}{@{}c@{}}Global ray \\ direction \end{tabular} } 
& 
\footnotesize{Local ray }
& 
\footnotesize{ \begin{tabular}{@{}c@{}}Local ray \\ direction \end{tabular} }
\\
\hline
$\mathbf{N}_{j,i} = \mathbf{f}_j-\mathbf{p}_i$
& 
$\mathbf{\breve{N}}_{j,i} = \frac{\mathbf{f}_j-\mathbf{p}_i}{\|\mathbf{f}_j-\mathbf{p}_i\|}$
& 
$\mathbf{N}_{j,i}^{(l)} = \mathbf{R}_i^\intercal(\mathbf{f}_j-\mathbf{p}_i)$
&
$\mathbf{\breve{N}}_{j,i}^{(l)} = \frac{\mathbf{R}_i^\intercal(\mathbf{f}_j-\mathbf{p}_i)}{\|\mathbf{R}_i^\intercal(\mathbf{f}_j-\mathbf{p}_i)\|}$ \\
\hline
 \end{tabular}
\end{table}
The information $\mathbf{N}_{i,j}$ encodes constitute constraints in a maximum a posterior (MAP) problem for poses and points.
\begin{equation}
\min_{ \mathbf{T}, \mathbf{f}  }\sum_{i,j} \| e_{ij}(\breve{\mathbf{N}}_{j,i}^{(l)}) \|^2 = \min_{ \mathbf{T}, \mathbf{f}  }\sum_{i,j} \|  e_{ij}(  \frac{\mathbf{R}_i^\intercal\mathbf{N}_{j,i} }{\| \mathbf{R}_i^\intercal\mathbf{N}_{j,i} \|}   ) \|^2.
\label{eq::NLS}
\end{equation}
In conventional BA, the error function $e_{ij}( \cdot )$ is given by:
\begin{equation}
e_{ij}( \mathbf{f}_j ) :=  \mathbf{K}\circ\pi(\mathbf{R}_i^\intercal
(\mathbf{f}_j-\mathbf{p}_i)) - \mathbf{u}_{m_{j,i}} \quad \in \mathbb{R}^2.
\label{eq::repro}
\end{equation}
BA with conventional parameterization and cost function suffers from the issues listed below: 

\begin{itemize}
\item Ill-conditioned case due to problematic features: Although these features still contain some information, they cause singularity in the Hessian matrix, a main contribution to GN divergence and numerical instability.
\item Slow convergence: To deal with singularity, slow LM is commonly used for safe increment, DL and GN are avoided, and efficiency is compromised for stability.
\item Stop criteria: Small changes in the error cost lead to large variation in the state variable, making it difficult to specify a \textbf{consistent} stop criterion.
\item Local minimum: the error function (\ref{eq::repro}) does not distinguish between in-frustum or behind camera features, thus causing many local minima and saddle points. A good example is the two-view geometry problem in which there are multiple global minima for the BA formulation such that further manual intervention is needed to pick the feasible solution.
\end{itemize}
In light of above discussion, safe-handling of ill-conditioned cases is vitally important for robustness, accuracy and efficiency of visual SLAM.

\section{Parallax Bundle Adjustment on Manifold}
\label{Section::PMBA}

In this section, we introduce the BA method using parallax angle in manifold domain (PMBA). We provide a thorough theoretical analysis on the boundedness of its information matrix, hence prove its smooth convergence without issues of singularity. We also show the error function is bounded and globally continuous. All these factors lead to possibility of faster optimization method DL, a significant improvement than  previous work \cite{PBAliangICRA}\cite{PBAliang}.

\subsection{Feature parameterization}\label{sec:pmba-feat}

\begin{figure}[h]
	\centering
\setlength\tabcolsep{1pt}
\begin{tabular}{cc}
\includegraphics[width=0.55\linewidth]{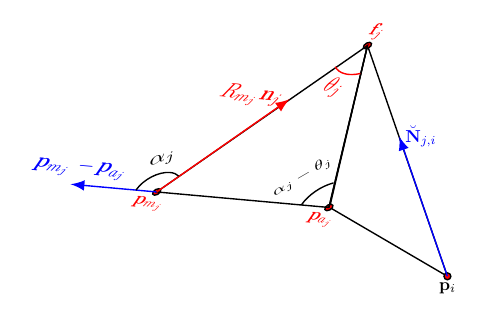} &
	\includegraphics[width=0.4\linewidth,height=3cm]{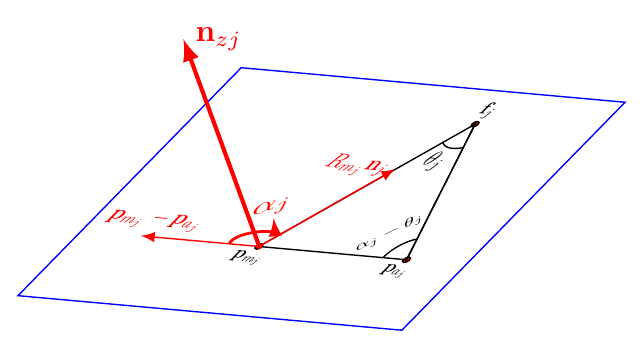} \\
	\footnotesize{ 
		\begin{tabular}{@{}l@{}}(\emph{a}) Feature $\mathbf{f}_j$ anchored by $\mathbf{p}_{m_j}$ and \\
		$\mathbf{p}_{a_j}$, $\theta_j$ is parallax angle between\\
		anchor rays, $\mathbf{n}_j$ is ray direction in\\main's frame, $\mathbf{p}_{i}$ is an arbitrary pose\\position co-visible for $\mathbf{f}_j$.
		\end{tabular} }
&
	\footnotesize{ 
			\begin{tabular}{@{}l@{}}(\emph{b}) Convexification of PMBA:  \\
			Rotate $( \mathbf{p}_a -\mathbf{p}_m  )$  about $\mathbf{n}_{zj}$ \\
			 by ($\pi - \alpha_j$)  becomes \\ $\| \mathbf{p}_m -\mathbf{p}_a \| \mathbf{R}_{m_j}$. \end{tabular}
			}
\end{tabular}
	\caption{The geometric structure about feature $j$ and its anchors in PMBA.}
	\label{fig::digram}
\end{figure}

A feature's depth information is implied in the parallax between observations from different viewpoints. For a feature $\mathbf{f}_j$, amongst the set of cameras $ \mathbb{T}_j $ to which $\mathbf{f}_j$ is visible, we choose a main anchor $\mathbf{T}_{ m_j }$ and an associate anchor $\mathbf{T}_{ a_j }$ that form best parallax angle from their observation rays. This geometric relationship among the feature $j$ is illustrated in Fig. \ref{fig::digram}\emph{(a)}.
The feature $\mathbf{f}_j$ can be over-parameterized by the unit observation ray vector $\mathbf{n}_j$ in main-anchor frame, and the parallax angle $\theta_j$, i.e.,
\begin{equation}
\mathbf{F}_j = ( \cos \theta_j,  \sin \theta_j, \mathbf{n}_j )
\label{eq::Fj}
\end{equation}
The new parameterization $\mathbf{F}_j$ only defines the relative structure of the feature with respect to its two anchors. The scale of the feature $\mathbf{f}_j$ is implicitly defined by the relative translation of the two anchors, computed as
\begin{equation}
\begin{aligned}
\mathbf{f}_j &= d_j  \mathbf{R}_{m_j}  \mathbf{n}_j + \mathbf{p}_{m_j} \\
&=\frac{\sin(\alpha_j-\theta_j)}{\sin(\theta_j)}  \|  \mathbf{p}_{m_j} - \mathbf{p}_{a_j}  \|\mathbf{R}_{m_j}  \mathbf{n}_j + \mathbf{p}_{m_j}
\label{eq::f}
\end{aligned}
\end{equation}
where 
\begin{itemize}
\item $d_j=  \frac{\sin(\alpha_j-\theta_j)}{\sin(\theta_j)}  \|  \mathbf{p}_{m_j} - \mathbf{p}_{a_j}  \|$
is the local depth of the feature $j$ in the main anchor frame, from sine rule.
\item $\mathbf{R}_{m_j}$ is the rotation for main anchor frame $\mathbf{T}_{m_j}$.
\item $\mathbf{n}_j \in \mathbb{R}^3 $ is the direction of observation ray from  point $\mathbf{f}_j$ to point $\mathbf{p}_{m_j}$, local in main anchor frame $\mathbf{T}_{m_j}$.
\item $\theta_j \in (0,\pi)$ is the parallax angle  between the vector $\mathbf{f}_j -\mathbf{p}_{m_j}$ and the vector $\mathbf{f}_j -\mathbf{p}_{a_j}$.
\item $ \alpha_j =  \arccos (  \frac{ (\mathbf{p}_{m_j} - \mathbf{p}_{a_j})^\intercal \mathbf{R}_{m_j}\mathbf{n}_j }{\| \mathbf{p}_{m_j} - \mathbf{p}_{a_j} \| }  )$  is the angle between vector $( \mathbf{p}_{m_j} - \mathbf{p}_{a_j})$ and vector $\mathbf{R}_{m_j} \mathbf{n}_j$.
\end{itemize}

\begin{remark}
In the original PBA parameterization \cite{PBAliang}, ray direction $\mathbf{n}_j$ was defined by an elevation and azimuth angle in the global frame, camera's orientation $\{\mathbf{R}_i\}$ in Euler angles. Expressing direction in sinusoids of angles is a potential source of singularity. In PMBA, both $\mathbf{n}_j$ and $\mathbf{R}_i$ are in the manifold domain. Moreoever, $\mathbf{n}_j$ is defined in $\mathbf{T}_{m_j}$'s local frame, for ease of multi-camera system application.
\end{remark}

\subsection{State retraction in manifold}

\begin{figure}[h]
	\centering
\includegraphics[width=0.8\linewidth,height=4.5cm]{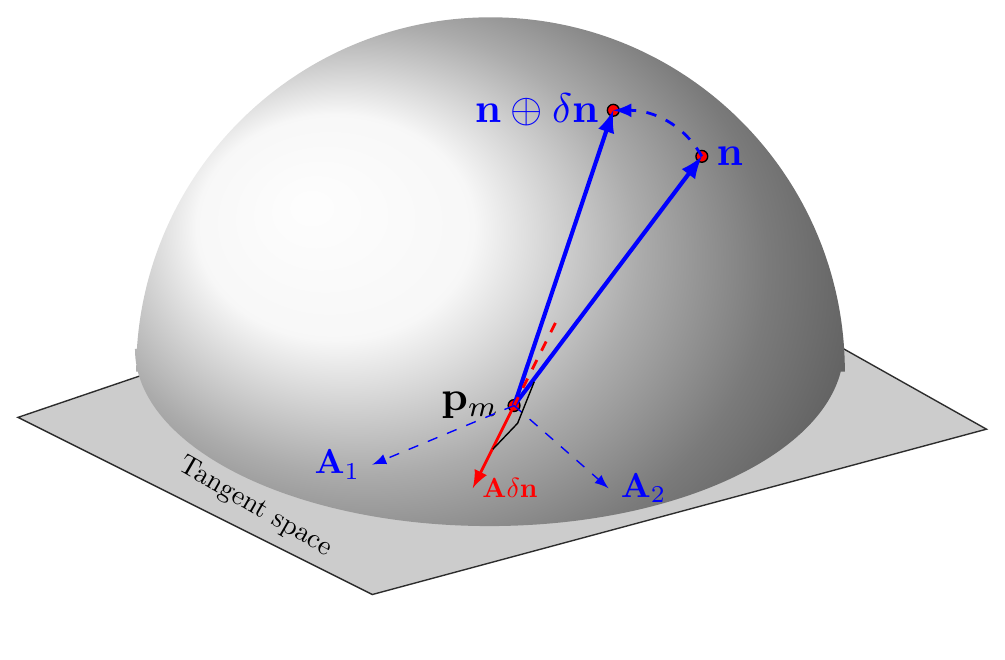}
	\caption{Retraction of ray $\mathbf{n}$ in main anchor}
	\label{fig::retraction}
\end{figure}

Optimization in manifold follows a 3 step procedure \cite{Manifold-lifting}: lift a manifold variable to its tangent space, solve a normal equation to obtain the Euclidean increment, and retract back to manifold. We adopt method in \cite{TengEKFSLAM}\cite{TengEKFVINS} for pose retraction. For feature's ray direction, we give a natural definition of uncertainty as a normally distributed rotational perturbation to the directional vector as shown in Fig.  \ref{fig::retraction}. The rotation's axis constitutes a plane normal to the ray passing through the observing camera, and is the tangent space, summarized in the following equation:
\begin{equation}
\tilde{\mathbf{n}}_j= \mathrm{Exp}(\mathbf{A}_{\mathbf{n}_j}{\delta\mathbf{n}_j})\mathbf{n}_j,\qquad \delta\mathbf{n}_j \in \mathcal{N}(0, \Sigma).
\label{eq::gaussianAxis}
\end{equation}
where $\delta\mathbf{n}_j \in \mathbb{R}^2$, $\mathbf{A}_{\mathbf{n}_j} \in \mathbb{R}^{3\times2}$ and  
$[ \mathbf{A}_{\mathbf{n}_j} \: \mathbf{n}_j ] \in \mathbb{SO}(3)$.
 The optimal perturbation is the increment for retraction $\oplus$:
 {\small
\begin{equation}
	\mathbf{F}_j \oplus \delta\mathbf{F}_j = ( \cos (\theta_j+\delta\theta_j),  \sin (\theta_j+\delta\theta_j), \mathrm{Exp}( \mathbf{A}_{\mathbf{n}_j} \delta\mathbf{n}_j  ) \mathbf{n}_j ). 	
	\label{eq::retractionF}
\end{equation}
}
 where the total increment $ \delta\mathbf{F}_j = \begin{bmatrix}
\delta\theta_j, \delta\mathbf{n}_j
\end{bmatrix} \in \mathbb{R}^3$ has same dimensionality as conventional parameterization.

\subsection{Error function and optimization formulation}

After determining the main anchor $\mathbf{T}_{m_j}$ and the associated $\mathbf{T}_{a_j}$ for each feature $j$, 
we can rewrite the nonlinear least squares problem (\ref{eq::NLS}) using the new feature parametrization
\begin{equation}
\min_{ \mathcal{X}  } \| f( \mathcal{X}  )  \|^2  = 
\min_{ \mathbf{T}, \mathbf{F}  }  \sum_{ i \in \mathbb{T}_j, j } \|  e_{ij}(  \frac{\mathbf{R}_i^\intercal \mathbf{N}_{j,i}}{\| \mathbf{R}_i^\intercal \mathbf{N}_{j,i} \|}   ) \|^2, \\
\label{eq::optimisation1}
\end{equation}
where   
 $\mathbf{F}= \{ \mathbf{F}_j \}_{j=1,\cdots,N}$ and $\mathcal{X}= (\mathbf{T}, \mathbf{F}) $. 
We now give ray $\mathbf{N}_{j,i}$ a new definition (with abuse of notation): the original ray vector scaled up by a factor of $\sin(\theta_j)$, for convenience of mathematical manipulation, i.e.,  
 \begin{equation}
\begin{aligned}
 	\mathbf{N}_{j,i} :=& \sin(\theta_j)(\mathbf{f}_j-\mathbf{p}_j)\\
 	=& \sin(\alpha_j -\theta_j ) \| \mathbf{p}_{m_j}-\mathbf{p}_{a_j}   \| \mathbf{R}_{m_j} \mathbf{n}_j \\
 	&  + \sin(\theta_j) (\mathbf{p}_{m_j} - \mathbf{p}_i).
\end{aligned}
 \end{equation}
 We also introduce a ray direction based error function, different from the conventional pixel-based error function (\ref{eq::repro}) (shown in Fig. \ref{fig::DirErr}), i.e., 
\begin{equation}
	\mathbf{e}_{ij}( \breve{\mathbf{N}}_{j,i}^{(l)}) := \breve{\mathbf{f}}^{(l)}_{m_{j,i}} - \breve{\mathbf{N}}_{j,i}^{(l)}  \in \mathbb{R}^3,
	\label{eq::errorfunction} 
\end{equation}
 where $\breve{\mathbf{f}}^{(l)}{m_{j,i}} =  \frac{\mathbf{K}^{-1} \mathbf{u}_{m_{j,i}}}{\| \mathbf{K}^{-1} \mathbf{u}_{m_{j,i}} \|}  \in \mathbb{R}^3$ is the measured directional vector for the feature $j$ in the pose $\mathbf{T}_i$. From now on, we use  $\breve{\mathbf{f}}^{(l)}_{j,i}$ to refer to $\breve{\mathbf{f}}^{(l)}{m_{j,i}}$ for simplicity.\\
 We further simplify (\ref{eq::optimisation1}) by moving measurement to global frame
\begin{equation}
\min_{ \mathcal{X}  } \| f( \mathcal{X}  )  \|^2  = 
\min_{ \mathbf{T}, \mathbf{F}  }  \sum_{ i \in \mathbb{T}_j, j } \| 
\breve{\mathbf{N}}_{j,i}- \mathbf{R}_i \breve{\mathbf{f}}_{j,i}^{(l)} \|^2. \\
\label{eq::optimisation2}
\end{equation}

\begin{remark}
	The error function (\ref{eq::errorfunction}) is globally continuous and its derivative is bounded, unlike the commonly used error function (\ref{eq::repro}). Its dimensionality is extended to 3D from 2D, meaning observation direction is also taken into consideration during optimization.
\end{remark}

Specifically,
\begin{itemize} 
\item The error equation (\ref{eq::errorfunction}) implies the residual $\|\mathbf{e}_{ij}\| = 2 \sin(\frac{\beta}{2})$,  where $\beta$ is the angle between the estimated and measured ray direction. Thus, the error equation is bounded.
\item In contrast to conventional 2D cost functions, our error function (\ref{eq::errorfunction}) operates in 3D and thus can handle the case when the feature point lies behind the observing camera.
\end{itemize}

\subsection{Theoretical analysis on behaviour of information matrix}

Consider the Hessian matrix of the problem (\ref{eq::optimisation1})  
\begin{equation}
	\mathbf{H} = \mathbf{J}^\intercal \mathbf{J}    = \begin{bmatrix}
	\mathbf{H}_{\mathbf{T}\mathbf{T}} & \mathbf{H}_{\mathbf{T}\mathbf{F}} \\
	\mathbf{H}^\intercal_{\mathbf{T}\mathbf{F}} & \mathbf{H}_{\mathbf{F}\mathbf{F}} \\
	\end{bmatrix},
\end{equation} 
where $\mathbf{J}:= \frac{\partial f(\mathcal{X}\oplus \Delta\mathcal{X})}{\partial \Delta\mathcal{X}}|_{\Delta\mathcal{X}= \mathbf{0}}$ and  $\mathcal{X}\oplus \Delta\mathcal{X} := ( \mathbf{T}\oplus \Delta\mathbf{T}, \mathbf{F}\oplus \Delta\mathbf{F}  )$.
Like the Hessian matrix in conventional BA, $\mathbf{H}_{\mathbf{FF}}$ is block diagonal. 
With the \textit{Schur's complement} method, the dominant computation in each Newton method's iteration is about solving
the following normal equation:
\begin{equation}
	( \mathbf{H}_{\mathbf{TT}} -  \mathbf{H}_{\mathbf{TF}} \mathbf{H}^{-1}_{\mathbf{FF}} \mathbf{H}^\intercal_{\mathbf{\mathbf{TF}}}    ) \Delta \mathbf{T} =   - \mathbf{C} f(\mathcal{X}),
\end{equation}
where $\mathbf{C} =  \begin{bmatrix}
\mathbf{I} & \mathbf{H}_{\mathbf{TF}} \mathbf{H}^{-1}_{\mathbf{FF}}
\end{bmatrix}$.
 In conventional BA, existence of problematic features makes the matrix $\mathbf{H}_{\mathbf{TT}} -  \mathbf{H}_{\mathbf{TF}} \mathbf{H}^{-1}_{\mathbf{FF}} \mathbf{H}^\intercal_{\mathbf{\mathbf{TF}}}$ and the block matrix $\mathbf{H}_{\mathbf{FF}}$ (with slight abuse of notation) ill-conditioned at the neighborhood of global minimum. The global minimum locates at  a ``long flat valley''  \cite{PBAliang} such that 
solvers fail or require long iterations to converge, see Fig. \ref{fig::converge}\emph{(a)} for illustration.

 \textbf{In comparison, PMBA's formulation (\ref{eq::optimisation2}),  thanks to the re-defined retraction (\ref{eq::retractionF}) and the error function (\ref{eq::errorfunction}),
 has an uncluttered Hessian, can therefore fully avoid the ill-conditioned cases caused by  ``problematic'' features.}    
\begin{theorem}
Under the formulation (\ref{eq::optimisation2}), 
	$\mathbf{H}_{\mathbf{FF}}$ is consistently non-singular for any $\mathcal{X}$ and $\mathbf{H}_{\mathbf{FF}} \geq \mathbf{I} $.
	\label{theorem::HFF}
\end{theorem}
\begin{proof}
	See Appendix \ref{appendix::theorem1}. 
\end{proof}
Theorem \ref{theorem::HFF}  completely suppresses  all ill-conditioned $\mathbf{H}_{\mathbf{FF}}$ such that achieving convergence becomes much eaiser. As a result, DL can be safely used for efficiency.
One can also appreciate Theorem \ref{theorem::HFF} from an Information Theory perspective: the Hessian matrix at global minimum is the inverse of the covariance matrix (up to a scale)  and thus
 the uncertainty of the parallax angle $\theta_j$ and the direction $\mathbf{n}_j$ 
is uniformly bounded.  

\begin{remark}
	 The original PBA \cite{PBAliang} cannot guarantee non-singularity in $\mathbf{H}_{\mathbf{FF}}$ due to
	  use of standard addition retraction for feature, Euler angles for orientation and the error function (\ref{eq::repro}).   
\end{remark}

\begin{remark}
	Although the  matrices $\mathbf{H}_{\mathbf{TT}}$ and $\mathbf{H}_{\mathbf{TF}}$ are denser, compared to those in XYZ or IDP, $\mathbf{H}_{\mathbf{TT}} -  \mathbf{H}_{\mathbf{TF}} \mathbf{H}^{-1}_{\mathbf{FF}} \mathbf{H}^\intercal_{\mathbf{\mathbf{TF}}}$ shows same sparsity. Thus the computational time for each iteration in PMBA is comparable to conventional BA, see \cite{PBAliang} for proof. 
\end{remark}

\section{Global Initialization}
\label{Section::Init}

In this section, we derive a novel initialization strategy. We do this in two steps: An orientations and parallax feature initialization step that involves cheap rotation averaging and anchor selection, without the need of expensive triangulation; then a translation-averaging method using a simplified convex pose-graph optimization. We prove that a near-optimal solution can be obtained by this strategy. This process is illustrated in Fig. \ref{fig::block-flow}. 

\begin{figure*}
	\centering
 	 \includegraphics[width=0.9\linewidth]{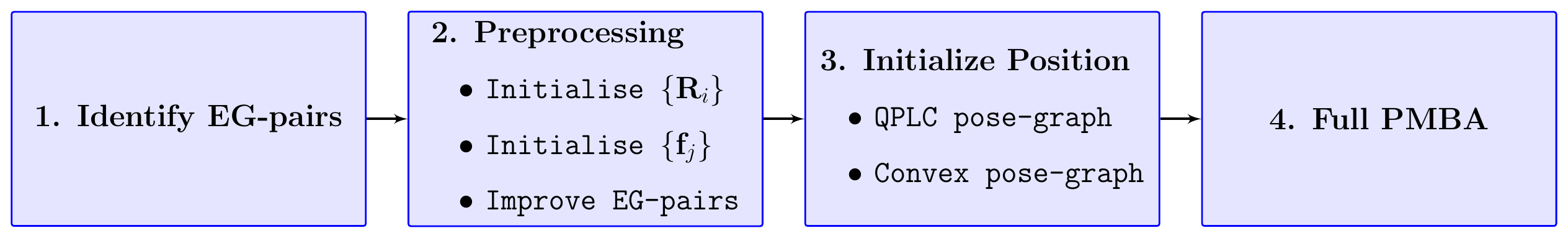}
  	\caption{
   		Full Global Initialization + PMBA pipeline.}
   	\label{fig::block-flow}
\end{figure*}

\subsection{Orientation and feature initialization}

Following the approach in \cite{cvprCrandall2011}\cite{Njiang}\cite{Wilson2014}, we first compute an initial guess for  orientation $\{ \mathbf{R}_i \}_{i=1,\cdots,M}$.
For each pair of pose $\mathbf{T}_i$ and $\mathbf{T}_k$ with common features above a threshold, we extract its Epipolar Geometry (EG) $( \tilde{\mathbf{R}}_{i,k}, \tilde{\mathbf{T}}_{i,k} )$ by Kneip's 5-point algorithm \cite{Kneip2012}.
We then use  the state-of-the-art chordal initialization \cite{LucaRotation} to accurately compute $\{ \bar{\mathbf{R}}_i \}$. We now feed $\{ \bar{\mathbf{R}}_i \}$ into OpenGV's two-pt ransac module \cite{OpenGV}, to obtain translation directions $\{ \bar{\mathbf{T}}_{i,k} \}_{i=1,\cdots,M}$.

 Having obtained accurate estimates for orientations and EG-pairs, we are ready to perform feature initialization. The  default anchor selection strategy was given in \cite{PBAliang}. We use the same algorithm for anchor selection, with the small change that co-visible pose scanned in pick anchors have to be part of an EG-pair. This step ensures best as-can-be parallax angle be given to each feature point. We stress that any problematic features corresponding to low parallax angles do stay in the state and do not affect convergence under PMBA. Good features together with problem ones work together to shape the final solution.

\begin{remark}
PMBA parameterization does not involve scale calculation, the selection algorithm in \cite{PBAliang} utilizes this property and only makes use of camera rotations to compute feature values in a fast and accurate way, we thus completely avoid unreliable/expensive linear triangulation.     
\end{remark}

\subsection{Position initialization}
After orientation and feature initialization, we can perform position initialization. We do this by approximating the original non-linear ray $\mathbf{N}_{j,i}$ function (\ref{eq::optimisation2}) with a linear relation of positions, helped with a rotation trick, as illustrated in Fig. \ref{fig::digram}\emph{(b)}. Now we give $\mathbf{N}_{j,i}$ a new formulation:
\begin{equation}
\begin{aligned}
\bar{\mathbf{N}}_{j,i} = &    \sin(\bar{\alpha}_j -\bar{\theta}_j )  \exp( \bar{\mathbf{n}}_{zj} (\pi -\bar{\alpha}_j)    ) ( \mathbf{p}_a -\mathbf{p}_m  ) \\
& 	+  \sin(\bar{\theta_j}) (\mathbf{p}_{m_j} - \mathbf{p}_i),
\end{aligned}
\label{eqn::convex}
\end{equation}
where
\begin{itemize}
\item $ \mathbf{n}_{zj} = \frac{\mathbf{p}_a -\mathbf{p}_m}{\| \mathbf{p}_a -\mathbf{p}_m \|} \times ( \mathbf{R}_{m_j} \mathbf{n}_j )  $  is the rotation axis from  the vector $\frac{\mathbf{p}_a -\mathbf{p}_m}{\| \mathbf{p}_a -\mathbf{p}_m \|}$ to the vector $ \mathbf{R}_{m_j} \mathbf{n}_j $.
\item $\alpha_{j} = \arccos(- \frac{\mathbf{p}_a -\mathbf{p}_m}{\| \mathbf{p}_a -\mathbf{p}_m \|} \cdot ( \mathbf{R}_{m_j} \mathbf{n}_j ))$ is the angle of rotation from vector $-\frac{\mathbf{p}_a -\mathbf{p}_m}{\| \mathbf{p}_a -\mathbf{p}_m \|}$ to the vector $ \mathbf{R}_{m_j} \mathbf{n}_j $.
\item Both $\mathbf{n}_{zj}$ and $\alpha_{j}$ are locally observable.
\item $\| \mathbf{p}_m -\mathbf{p}_a \| \mathbf{R}_{m_j} \mathbf{n}_j \equiv  \mathrm{Exp}( \mathbf{n}_{zj} (\pi - \alpha_j )    ) ( \mathbf{p}_a -\mathbf{p}_m  ) $ 	 is rotation of vector $(\mathbf{p}_a -\mathbf{p}_m)$ about axis $\mathbf{n}_{zj}$ by $\pi -\alpha_j$ angle.
\end{itemize}
Inspired by the translation averaging method in \cite{Wilson2014}, we obtain a ``\textit{position only}'' convex cost function, after substituting  $\bar{\mathbf{N}}_{j,i}$ into the original formulation (\ref{eq::optimisation2}) optimizationm:
\begin{equation}
\min_{ \{ \mathbf{p} \} }  h( \mathbf{p}, \bar{\mathbf{R}}, \bar{\mathbf{F}} )  :=
\min_{ \{ \mathbf{p} \} }  \sum_{ i \in \mathbb{T}_j, j } \| 
\breve{\bar{\mathbf{N}}}_{j,i}
- \bar{\mathbf{R}}_i \breve{\mathbf{f}}_{j,i}^{(l)} \|^2. \\
\label{eq::optimisation3}
\end{equation}

\begin{remark}
Considering (\ref{eq::optimisation3}) is still a nonlinear problem, an initial guess for the problem (\ref{eq::optimisation3}) is needed. Since $\bar{\mathbf{N}}_{j,i}$ is linear to positions,
 we further simplify (\ref{eq::optimisation3}) to a linearly constrained Quadratic Programming (QPLC) problem: to minimize the cross-product between ray $N_{j,i}$ and $\mathbf{R}_i\breve{\mathbf{f}}_{j,i}$, with a linear constraint to ensure local observation ray $\mathbf{R}_i{\mathbf{N}}_{j,i}$ lies in front of the camera, as shown in Fig. \ref{fig::DirErr}\emph{(b)}, i.e.,
\begin{equation}
	\min_{ \{ \mathbf{p} \} }  \sum_{ i \in \mathbb{T}_j, j } \|  S(\bar{\mathbf{R}}_i \breve{\mathbf{f}}_{j,i}^{(l)}) \bar{\mathbf{N}}_{j,i} \|^2,\quad  z(\bar{R}_i\bar{\mathbf{N}}_{j,i}) >= 0.
\end{equation}  
\label{eq::optimisation4}
\end{remark}

\subsection{Theoretical analysis}

\begin{theorem}
With accurate initial estimate for orientation, the formulation  (\ref{eq::optimisation3}) can always converge to a near-optimal solution for both problem (\ref{eq::optimisation2}) and (\ref{eq::optimisation3}). Furthermore,  the problem (\ref{eq::optimisation3}) is convex when EG pairs are noise-free.
\label{therem::Convex}
\end{theorem}
\begin{proof}
	See Appendix \ref{appendix::theorem2}. 
\end{proof}
Theorem \ref{therem::Convex} proves the correctness and robustness of the proposed initialization in theory. Moreover, (\ref{eq::optimisation3}) is a pose-graph problem with much reduced size than (\ref{eq::optimisation2}) and the expensive feature retraction operation is also not needed.
 

\begin{remark}
	Here we do not claim the proposed global initialization is the best one but it is very compatible to PMBA. 
	Note that the proposed method is  friendly to robust methods such as pseudo Huber, $L^1$-norm or outlier detection  technique. 
	Further, this convexified model is still formulated in a probabilistic framework, different from the ``Linear Global Translation Estimation'' reported in \cite{LinearGlobalTranslation}. 
\end{remark}

\section{Evaluation on PMBA performance}\label{Section::Experiment}
\subsection{Simulation}
We demonstrate PMBA's ability to handle problem features by running a simple simulation test: 4 poses and 10 features. One of the problem features is a far feature, another is a singular feature that would cause singularity in original PBA, as shown in Fig. \ref{fig::simu}\emph{(a)}. We run 4 iterations for the BAs under comparison: DL for PMBA and XYZ-BA; and LM for PBA. At each iteration, we collect the Hessian's condition number, and at the end report the error between optimization results and ground truth. The results are listed in Table \ref{table::simu}. PMBA has normal condition numbers and gave good optimized estimates, PBA and XYZ-BA show consistently large condition numbers and high final state error. This confirms our prediction that PMBA has well-behaved information matrix during optimization.

\begin{figure}[h]
	\centering
	\setlength\tabcolsep{1pt}
	\begin{tabular}{cc}
		\includegraphics[width=0.49\linewidth,height=3cm]{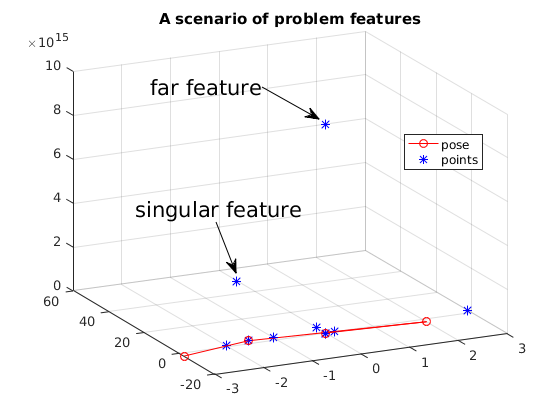}&
		\includegraphics[width=0.49\linewidth,height=3cm]{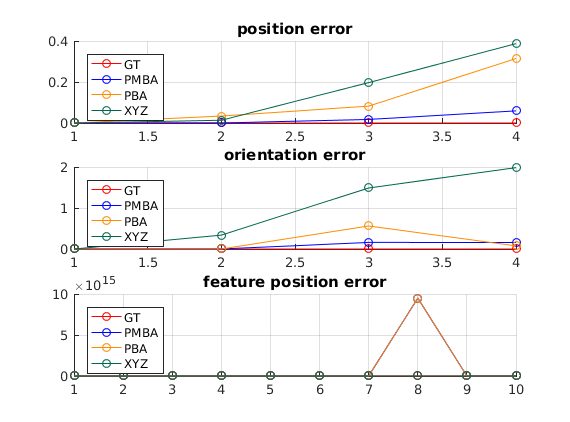} 
		\\
		\footnotesize{\begin{tabular}{@{}l@{}}(\emph{a}) Simulation with\\ problem features\end{tabular}} & \footnotesize{\begin{tabular}{@{}l@{}}(\emph{b}) Comparison of the error\\ from PMBA, PBA and XYZ-BA\end{tabular}}\\
	\end{tabular}
	\caption{\footnotesize{Compare three BA forms in a simulated scene with problem features}}
	\label{fig::simu}
\end{figure}

\begin{table}[h]
\centering
	\caption{\footnotesize Comparison of $\mathbf{H}_{\mathbf{FF}}$'s condition number during optimization and final state error for PMBA, PBA and XYZ-BA}
	\label{table::simu}
	\begin{tabular}{|c|c|c|c|}
\hline
 Convergence Properties & PMBA    &  PBA & XYZ-BA  \\
 \hline
Iter-0 $cond(\mathbf{H}_{\mathbf{FF}})$  & 9.74 & 1.46E+4 & 1.22E+94 \\
Iter-1 $cond(\mathbf{H}_{\mathbf{FF}})$ & 5.68 & 1.46E+4 & 1.22E+94 \\
Iter-2 $cond(\mathbf{H}_{\mathbf{FF}})$ & 8.80 & 1.46E+4 & 1.22E+94 \\
Iter-3 $cond(\mathbf{H}_{\mathbf{FF}})$ & 5.74 & 1.46E+4 & 3.53E+95 \\
\hline
Final $\chi_{error}^2$ & 2.58E-3 & 5.37E-2 & 3.43E-2 \\
\hline
\end{tabular}
\end{table}


\subsection{Large dataset test}\label{SubSection::test-convergence}
We conducted a series of real datasets to compare performance of the proposed PMBA (\ref{eq::optimisation2}) and original PBA, IDP and XYZ, aiming to address following questions:
\begin{itemize}
	\item \textbf{Robustness.} 
	With ill-conditioned scenario disappearing, can DL be safely used in PMBA?
	
	\item  \textbf{Efficiency.}
	If DL were safely applied for PMBA formulation, how fast can the optimization process be?  
		
	\item  \textbf{Accuracy.} 
	Since the PMBA formulation employs a different error function (\ref{eq::errorfunction}). 
	Is the global minimum accurate?	
\end{itemize}

All methods are tested against 6 very challenging datasets, which are also accessible from OpenSLAM\footnote{\url{https://svn.openslam.org/data/svn/ParallaxBA/}}. In particular,
\begin{itemize}
\item \textit{Fake-pile} is collected by the Google tango tablet in normal lab environment \cite{gtango} with a fake bridge pile in the middle, showing close and far features.
\item \textit{Malaga} \cite{Malaga} is collected using an electric car equipped camera facing the road, rich in collinear features.
\item \textit{Village} and \textit{College} are aerial photogrammetric datasets. The low feature to observation ratio implies the existence of many small parallax features
\item \textit{Usyd-Mainquad-2} and \textit{Victoria-cottage} are collected at University of Sydney campus, show in Fig. \ref{fig::usyd-mainquad-2}.

\end{itemize}

\begin{figure}
	\centering
	\setlength\tabcolsep{1pt}
	\begin{tabular}{ll}
 \includegraphics[width=0.47\linewidth,height=3cm]{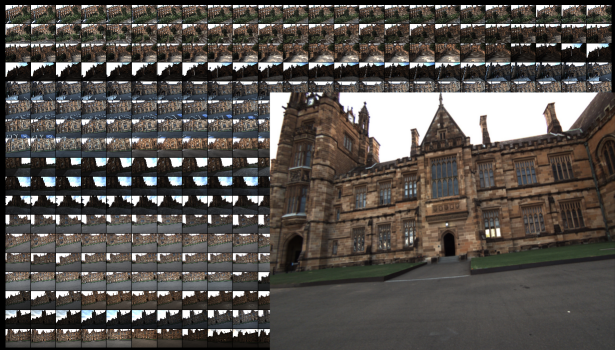} &   \includegraphics[width=0.47\linewidth,height=3cm]{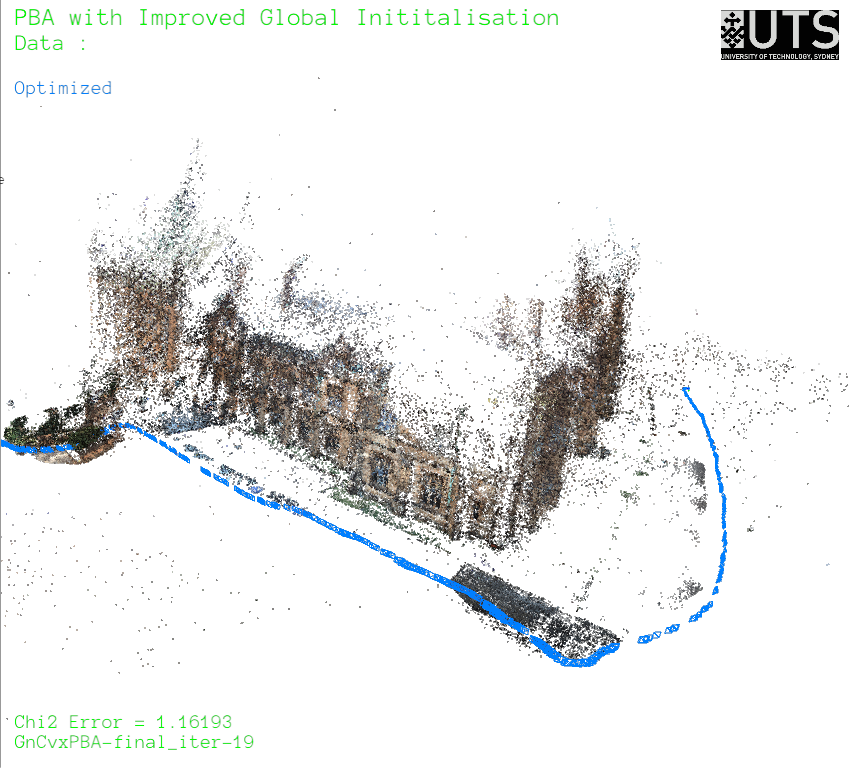} \\		
\footnotesize{(\emph{a}) Images } & \footnotesize{(\emph{b}) PMBA output pose and map}\\ 
\end{tabular}
\caption{Usyd-mainquad dataset}
\label{fig::usyd-mainquad-2} 
\end{figure}

\begin{table*}[]
	\centering
	\setlength\tabcolsep{3pt}	
\caption{\footnotesize Comparison of convergence performance for PMBA, PBA, XYZ-BA, IDP-BA}
	\label{table::cmprperfm}
	\begin{tabular}{|c|c|c|c|c|c|c|}
		\hline
		\footnotesize{Dataset} & \footnotesize{Test-type} &
		\footnotesize{\begin{tabular}{@{}c@{}}\# Pose \\ / \# Feat \\ / \# Obsv \end{tabular}
		} &
		 \footnotesize{\begin{tabular}{@{}c@{}}\# Equation solving \\ / \# Iteration \end{tabular}} & \footnotesize{Initial Chi2} & \footnotesize{Final Chi2} & \footnotesize{Time[sec]} \\
		\hline
	\footnotesize{Fake-pile} & \footnotesize{PMBA} & \footnotesize{135} & \footnotesize{9 / 9} & \footnotesize{6.6E+6}  & \footnotesize{1.7E+2}  & \footnotesize{0.7} \\
    & \footnotesize{PBA} & \footnotesize{/12,741} & \footnotesize{23 / 23} & \footnotesize{6.6E+6} & \footnotesize{1.7E+2}  & \footnotesize{1.9}
	    \\
	 & \footnotesize{IDP} & \footnotesize{/53,878} & \footnotesize{104 / 102} & \footnotesize{6.6E+6}  & \footnotesize{1.7E+2} & \footnotesize{6.0}
	    \\
	 & \footnotesize{XYZ} &  & \footnotesize{116 / 108} & \footnotesize{6.6E+6} & \footnotesize{1.2E+3} & \footnotesize{4.7} \\
		\hline		
		
	\footnotesize{Malaga} & \footnotesize{PMBA} & \footnotesize{170} & \footnotesize{44 / 31} & \footnotesize{3.8E+7} & \footnotesize{9.1E+3}
	 & \footnotesize{21.6} \\
	 & \footnotesize{PBA} & \footnotesize{/305,719} & \footnotesize{64 / 47} & \footnotesize{3.8E+7} 
	  & \footnotesize{9.1E+3}
	   & \footnotesize{35.4} \\
	 & \footnotesize{IDP} & \footnotesize{/779,268} & \footnotesize{230 / 170} & \footnotesize{3.1E+7}	  & \footnotesize{5.8E+5}	   & \footnotesize{93.8}	    \\
	 & \footnotesize{XYZ} &  & \footnotesize{110 / 85} & \footnotesize{3.8E+7}
	  & \footnotesize{3.3E+5}
	   & \footnotesize{39.0}	    \\
		\hline
		
    \footnotesize{Village} & \footnotesize{PMBA} & \footnotesize{90} & \footnotesize{12 / 12} & \footnotesize{1.0E+10}
     & \footnotesize{3.3E+4} & \footnotesize{31.8} \\
	 & \footnotesize{PBA} & \footnotesize{/305,719} & \footnotesize{13 / 13} & \footnotesize{1.0E+10}
	  & \footnotesize{3.3E+4}
	   & \footnotesize{36.0}	    \\
	 & \footnotesize{IDP} & \footnotesize{/779,268} & \footnotesize{19 / 19} & \footnotesize{1.0E+10} & \footnotesize{3.3E+4} & \footnotesize{35.2} \\
	 & \footnotesize{XYZ} &  & \footnotesize{18 / 18} & \footnotesize{1.0E+10} & \footnotesize{3.3E+4} & \footnotesize{26.3} \\
		\hline
		
    \footnotesize{College} & \footnotesize{PMBA} & \footnotesize{468} & \footnotesize{33 / 33} & \footnotesize{3.0E+11} & \footnotesize{1.1E+6} & \footnotesize{334.4} \\
	 & \footnotesize{PBA} & \footnotesize{/1,236,502} & \footnotesize{31 / 31} & \footnotesize{3.0E+11} & \footnotesize{1.1E+6} & \footnotesize{370.5} \\
	 & \footnotesize{IDP} & \footnotesize{/3,107,524} & \footnotesize{34 / 34} & \footnotesize{3.0E+11} & \footnotesize{1.1E+6} & \footnotesize{255.3} \\
	 & \footnotesize{XYZ} &  & \footnotesize{295 / 193} & \footnotesize{3.0E+11} & \footnotesize{1.0E+7} & \footnotesize{1361.0} \\
		\hline
		
    \footnotesize{Victoria} & \footnotesize{PMBA} & \footnotesize{400} & \footnotesize{19 / 16} & \footnotesize{6.2E+8} & \footnotesize{1.1E+6} & \footnotesize{70.5} \\
	 \footnotesize{cottage}& \footnotesize{PBA} & \footnotesize{/153,632} & \footnotesize{88 / 66} & \footnotesize{6.2E+8} & \footnotesize{1.2E+6} & \footnotesize{301.4} \\
	 & \footnotesize{IDP} & \footnotesize{/890,057} & \footnotesize{49/48} & \footnotesize{6.2E+8} & \footnotesize{1.1E+6} & \footnotesize{157.9} \\
	 & \footnotesize{XYZ} &  & \footnotesize{47 / 44} & \footnotesize{6.2E+8} & \footnotesize{1.2E+6} & \footnotesize{124.3} \\
		\hline

    \footnotesize{Usyd} & \footnotesize{PMBA} & \footnotesize{424} & \footnotesize{25 / 25} & \footnotesize{2.4E+9} & \footnotesize{2.4E+6} & \footnotesize{214.5} \\
	 \footnotesize{-Mainquad}& \footnotesize{PBA} & \footnotesize{/227,615} & \footnotesize{101 / 57} & \footnotesize{2.4E+9} & \footnotesize{3.6E+6} & \footnotesize{642.6} \\
	& \footnotesize{IDP} & \footnotesize{/1,607,082} & \footnotesize{301 / 191} & \footnotesize{2.2E+9} & \footnotesize{4.6E+6} & \footnotesize{1994.7} \\
	 & \footnotesize{XYZ} &  & \footnotesize{76 / 58} & \footnotesize{2.4E+9} & \footnotesize{2.8E+6} & \footnotesize{423.7} \\
		\hline											\hline		
	\end{tabular}
\end{table*}

\begin{figure}[h]
	\centering
	\setlength\tabcolsep{1pt}
	\begin{tabular}{cc}
		\includegraphics[width=0.48\linewidth]{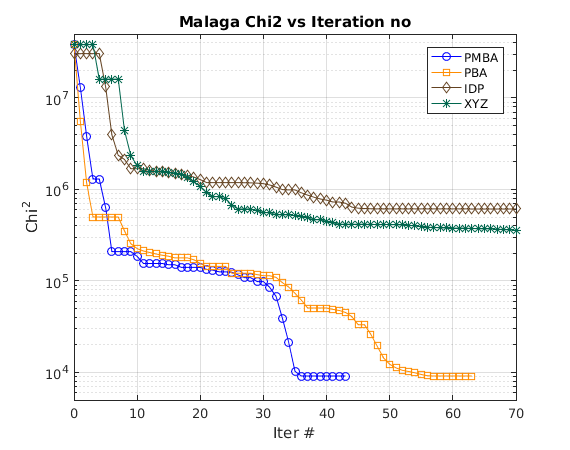} &		
	\includegraphics[width=0.48\linewidth]{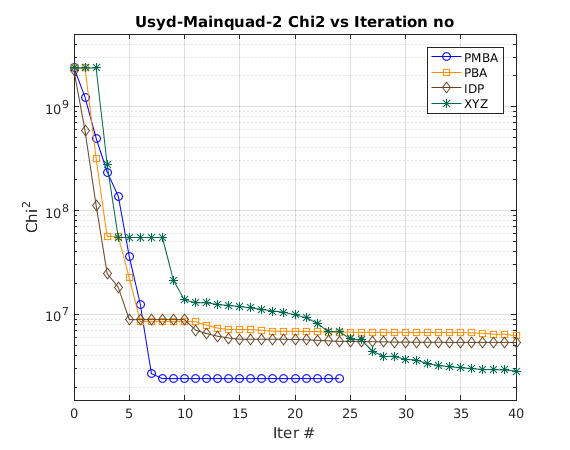} \\
	\tiny{(\emph{a}) Malaga }	&
	\tiny{(\emph{b}) Usyd-Mainquad-2} \\
	\includegraphics[width=0.48\linewidth]{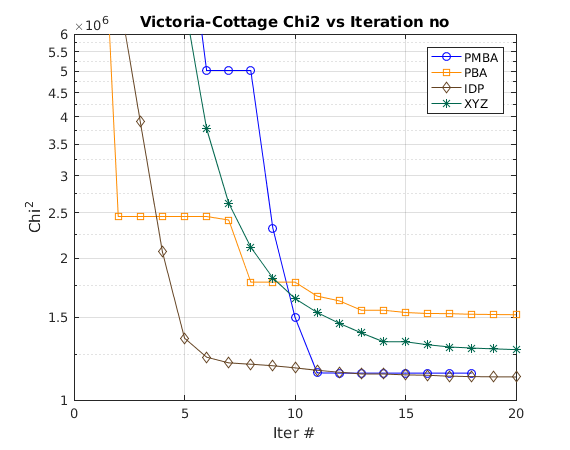} &
		\includegraphics[width=0.48\linewidth]{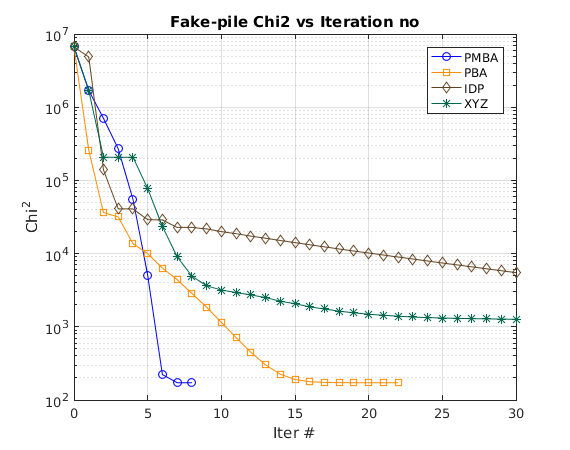} \\
	\tiny{ (\emph{c}) Victoria-cottage} &
	\tiny{(\emph{d}) Fake-pile} \\                       
	\end{tabular}
\caption{Convergence plots for PMBA, PBA, IDP and XYZ}
\label{fig::converge}	
\end{figure}

We use the initialization method from \cite{PBAliang} to set all BAs from the same starting point. We find that PBA, IDP and XYZ show unstable behaviour under DL. PMBA, in comparison, has always worked well with DL. This can be explained by our Theorem \ref{theorem::HFF} that Hessian in PMBA does not exhibit singularity yet others can, we therefore list DL results for PMBA and LM for other BA's. 


We use Ceres-solver as the optimization engine and test all BAs on an Intel-i7 with one thread. We use ray direction cost function for PMBA, and compute its corresponding uv-based Chi2 error at each iteration step with current estimate, to  compare with other BAs on a common error metric. This scheme is not fair for PMBA, yet is the only convincing way to evaluate performance amongst all methods. Despite of this treatment, we found PMBA the best performer in all tests, consistent to our expectation.
Selected convergence plots are shown in Fig. \ref{fig::converge}, other details are summarized in Table \ref{table::cmprperfm}.

Further, in the Malaga dataset that contains numerous problematic features (Fig. \ref{fig::malaga-collinear}), we observe that the PMBA estimates and Ground Truth are very close, yet conventional BA gives significant error. This is also seen in Table \ref{table::cmprperfm}, conventional BA's converge to a local minimum, whereas both PBA and PMBA can converge to their respective global minimums. Fig. \ref{fig::converge} demonstrates the error function (\ref{eq::errorfunction}) is practical,  consistent with  the claim in \cite{Im2016}. In conclusion, these experiments all give positive answers to the raised questions.

\subsection{Evaluation of convexified initialization}
In this subsection, we use selected datasets from the ``Bundle Adjustment in the Large'' (BAL) datebase\footnote{\url{http://grail.cs.washington.edu/projects/bal/}} \cite{BAL} and the datasets in Section \ref{SubSection::test-convergence} to verify our initialization strategy.
We implement a SfM pipeline according to the procedure in Fig. \ref{fig::block-flow}. The QPLC
stage was implemented with matlab toolbox quadprog, the rest with Ceres in C++. We present results for following datasets:
\begin{itemize}
\item Ladybug-1370: 1370 images, captured at a regular rate using a Ladybug camera mounted on a moving vehicle.
\item Trafalgar-126: 126 out-of-order internet images.
\item Venice-427: 427 out-of-order internet images.
\item College: 468 arial photogrammetric images.
\end{itemize}
 
All these datasets include either collinear or far features, exposing challenges for conventional BA. Since camera calibration is beyond the scope of this activity, we apply the reported optimal camera settings from BAL and PBA websites and only test undistorted versions of these data. We stress that our initial pose and feature values are purely generated from the Rotation averaging and Translation averaging method described in Section \ref{Section::Init}, without using the initial values provided by \cite{BAL}.
We are able to form good initial values at QPLC stage, shown in Fig. \ref{fig::qpcvxpmba-pipeline} as 1st image in each row. And, our convex pose-graph stage has a very large convergence region such that imperfect outputs from the QPLC stage can gradually converge to a pose-graph with a topology similar to that of Ground Truth. This is especially obvious in ``Ladybug-1370'' and ``Venice-427''. Moreover, in ``Ladybug-1370'', BAL's optimal trajectory (in red) contains an erroneous camera pose, shown as red dot at top right corner of the red trajectory, our method did not encounter this stray pose at all.

 \begin{figure}{}
	\centering
		\setlength\tabcolsep{1pt}
			\begin{tabular}{cccc}
 	\includegraphics[width=0.24\linewidth,height=2cm]{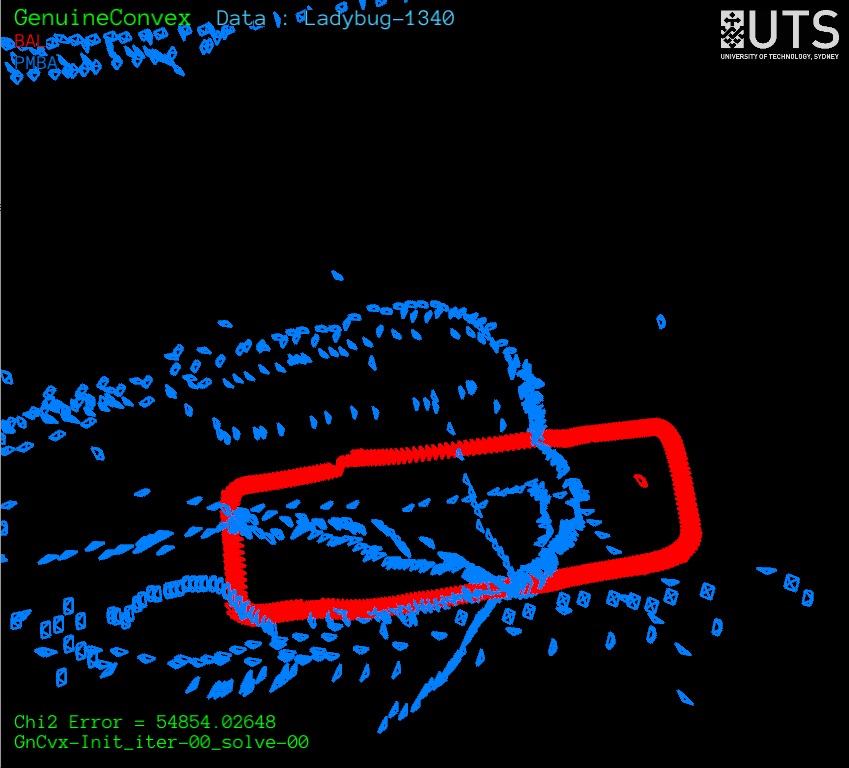} &
 	\includegraphics[width=0.24\linewidth,height=2cm]{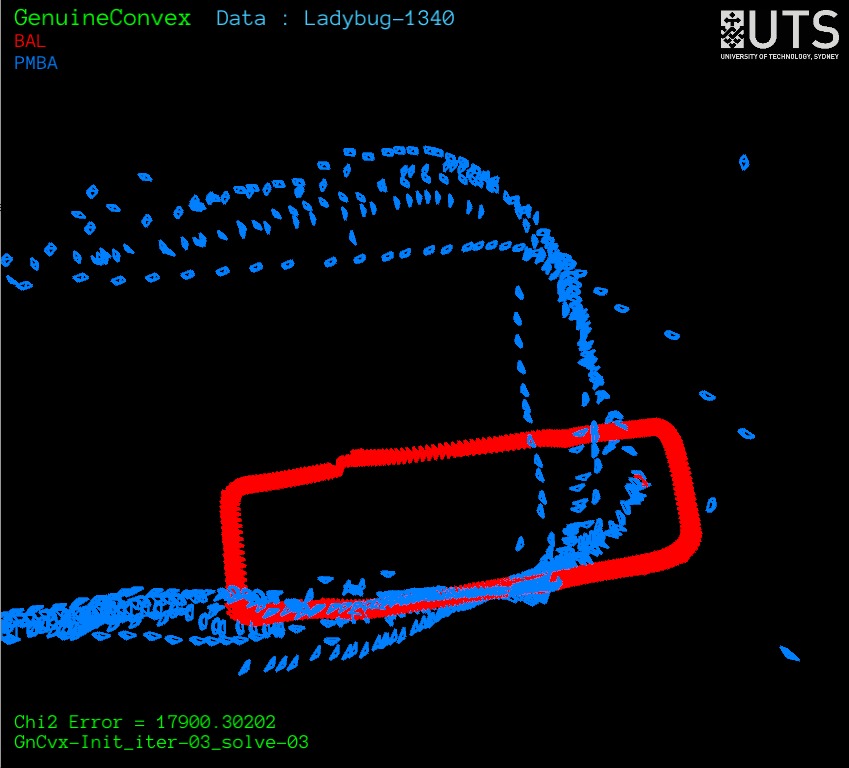} &
 	\includegraphics[width=0.24\linewidth,height=2cm]{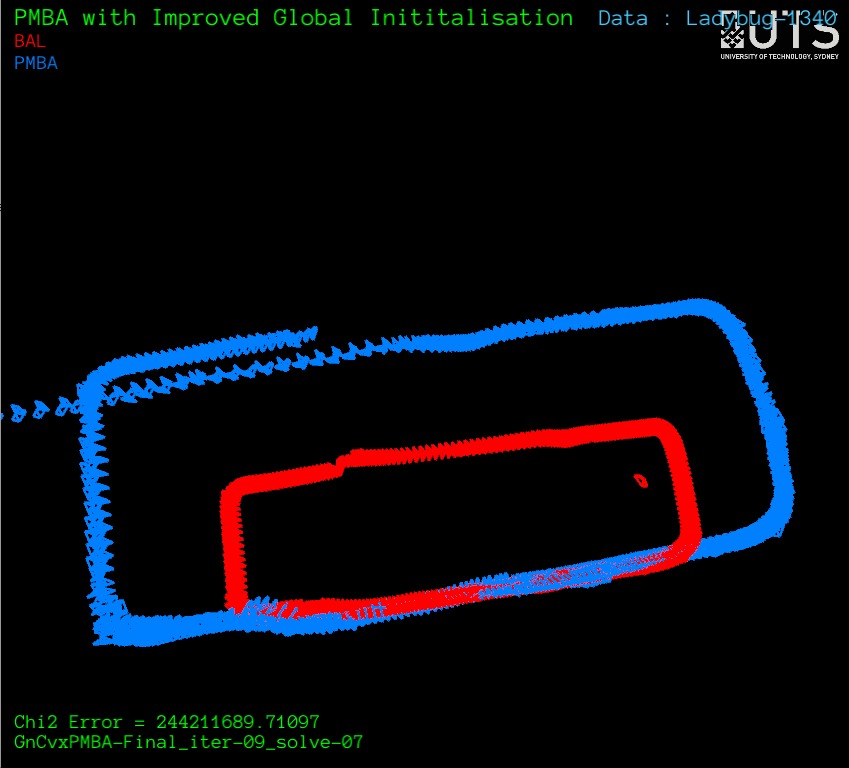} &
 	\includegraphics[width=0.24\linewidth,height=2cm]{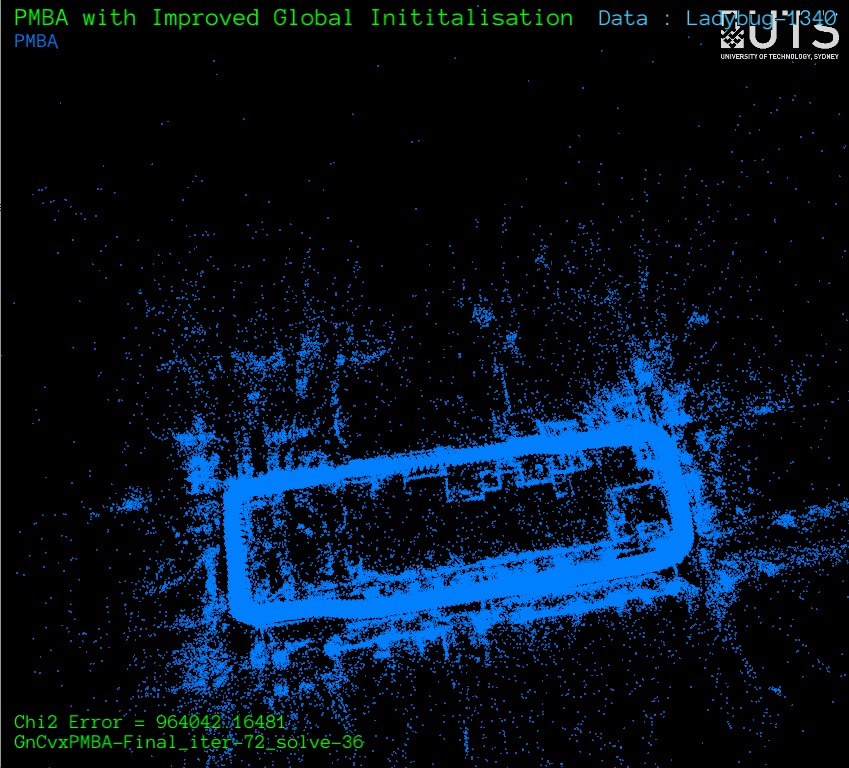} \\ 	 	 	
   	\includegraphics[width=0.24\linewidth,height=2cm]{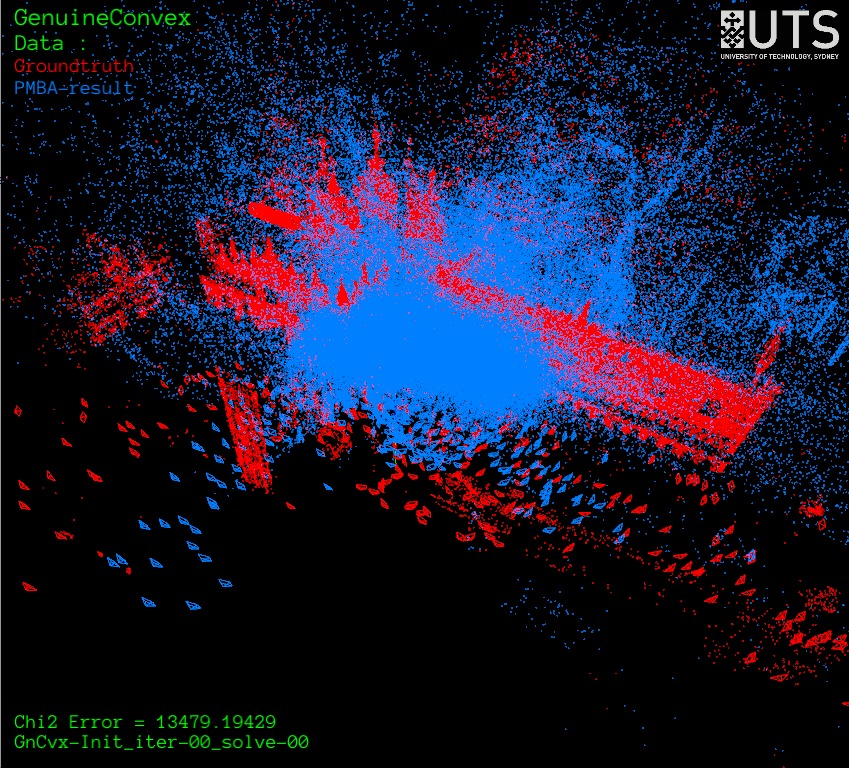} &
   	\includegraphics[width=0.24\linewidth,height=2cm]{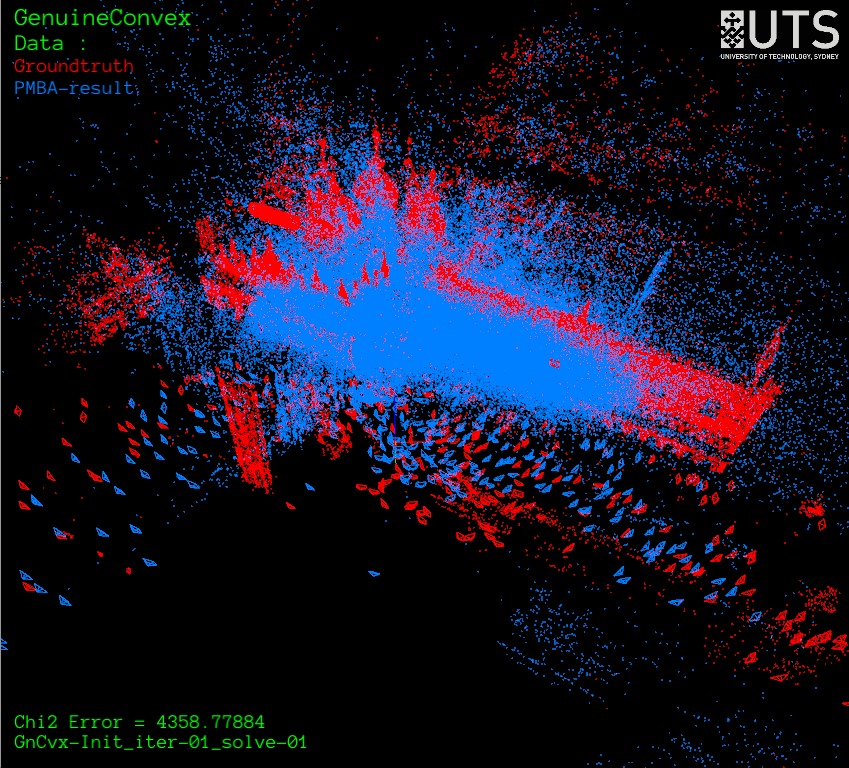} &
   	\includegraphics[width=0.24\linewidth,height=2cm]{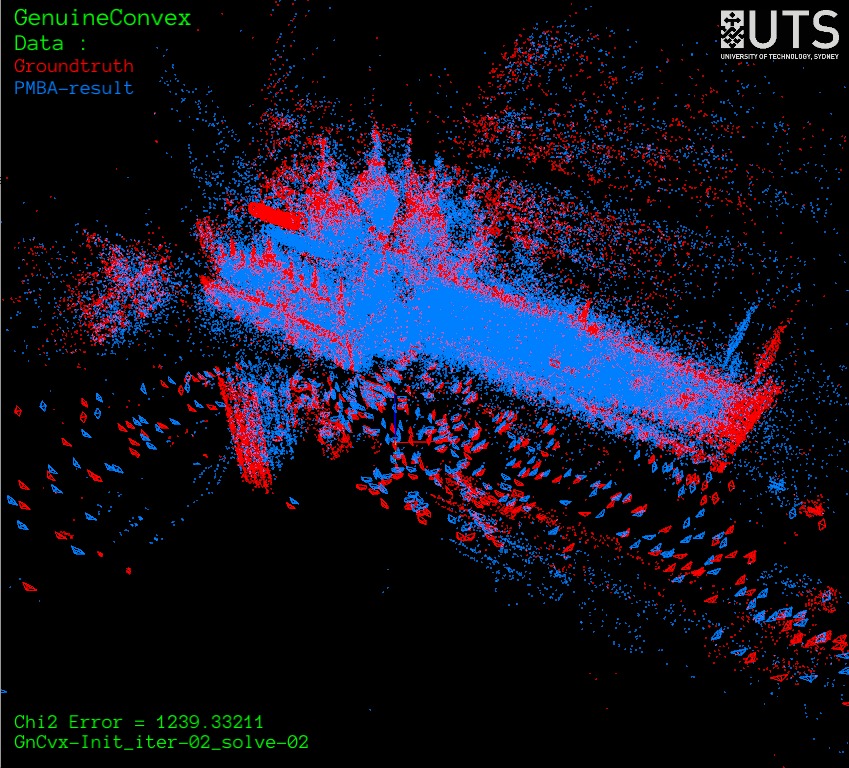} &
   	\includegraphics[width=0.24\linewidth,height=2cm]{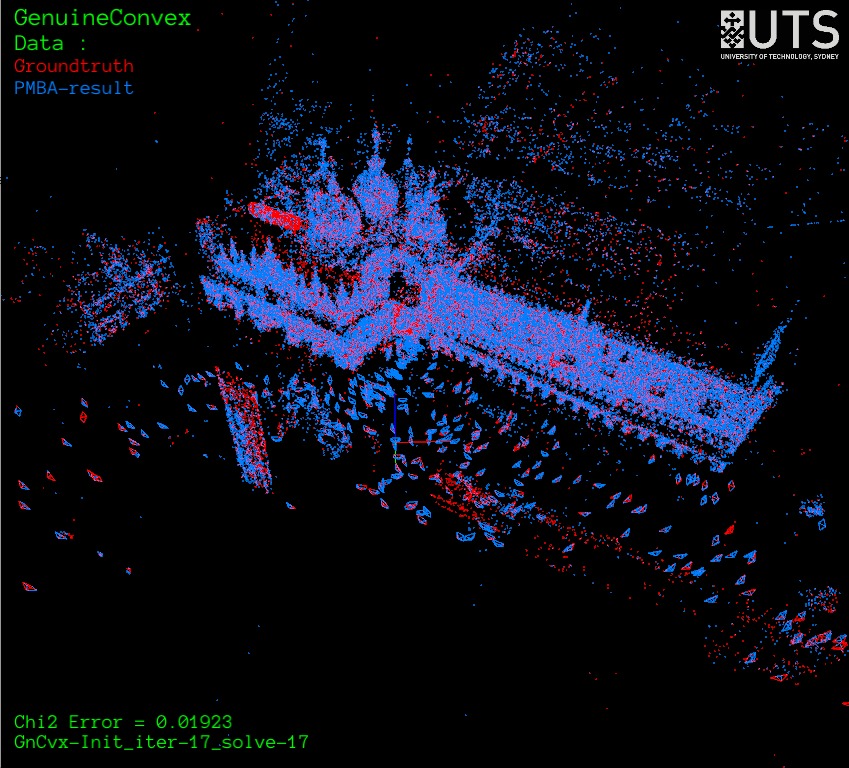} 
   	\\
   	\includegraphics[width=0.24\linewidth,height=2cm]{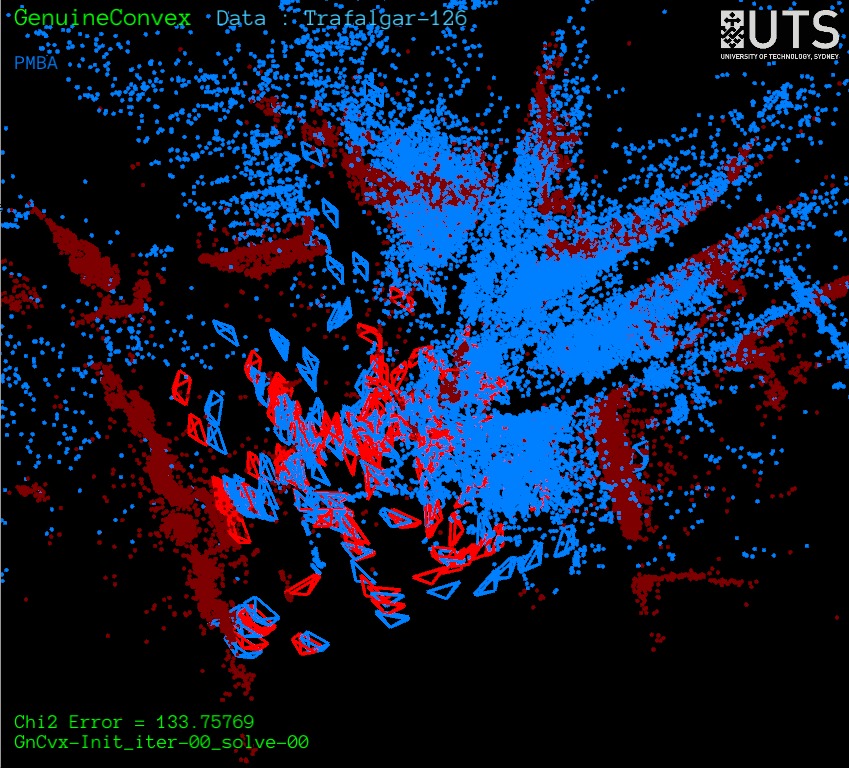} &
   	\includegraphics[width=0.24\linewidth,height=2cm]{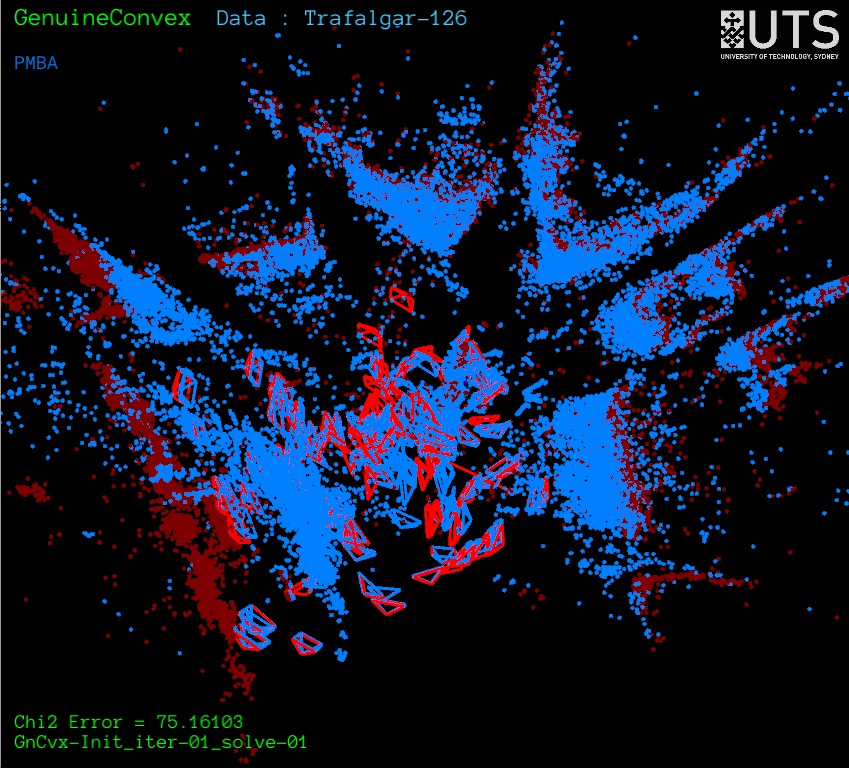} &
   	\includegraphics[width=0.24\linewidth,height=2cm]{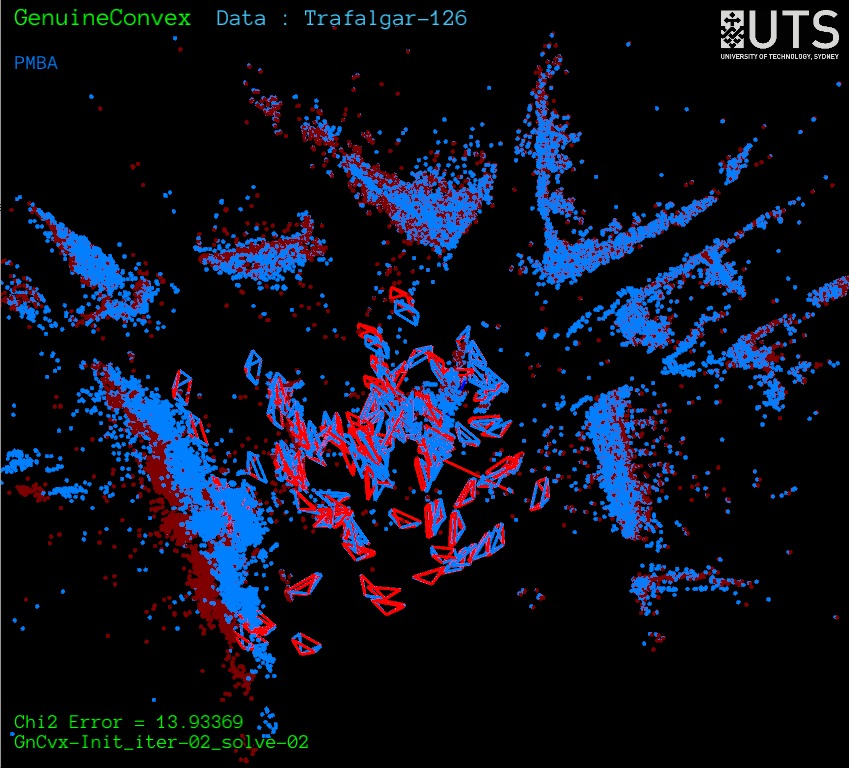} &
   	\includegraphics[width=0.24\linewidth,height=2cm]{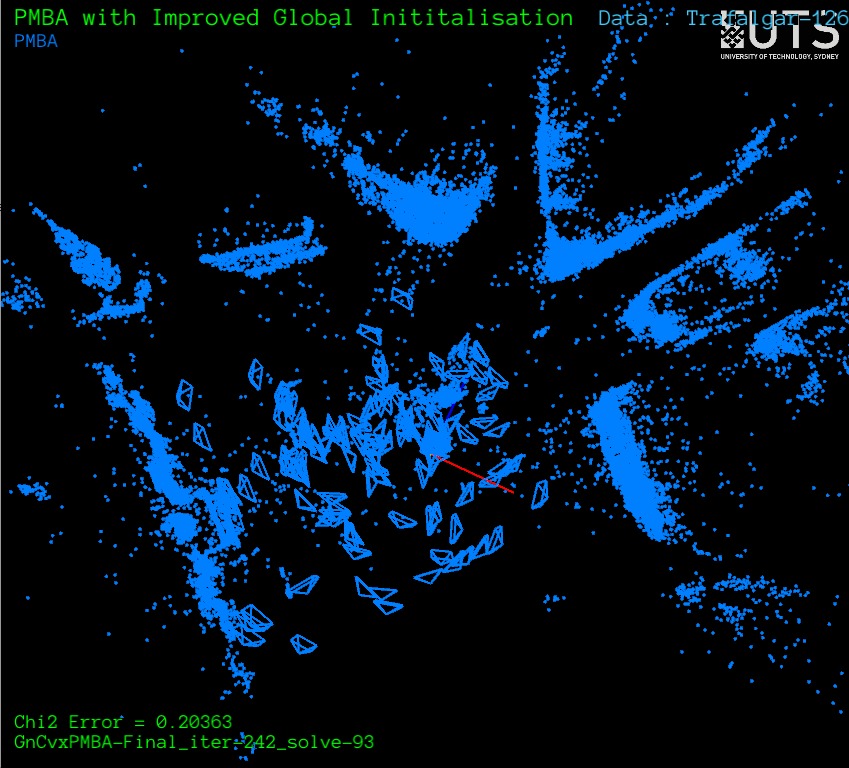} \\   	
   	\includegraphics[width=0.24\linewidth,height=2cm]{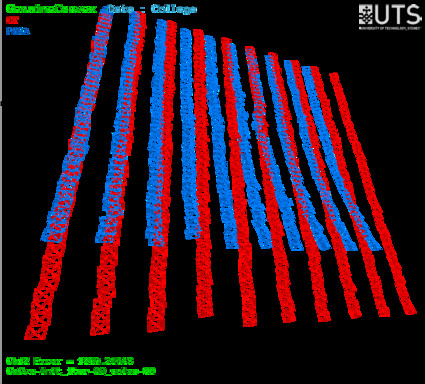} &
   	\includegraphics[width=0.24\linewidth,height=2cm]{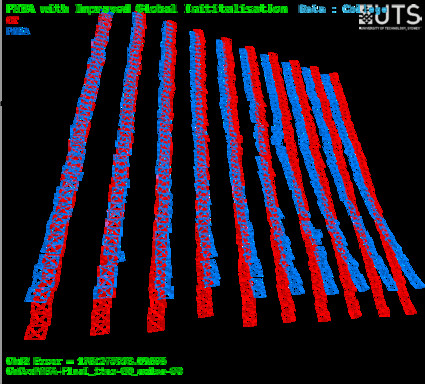} &
   	\includegraphics[width=0.24\linewidth,height=2cm]{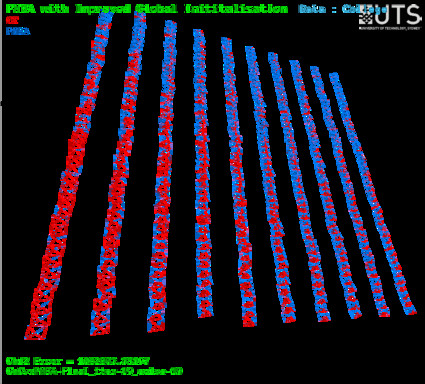} &
   	\includegraphics[width=0.24\linewidth,height=2cm]{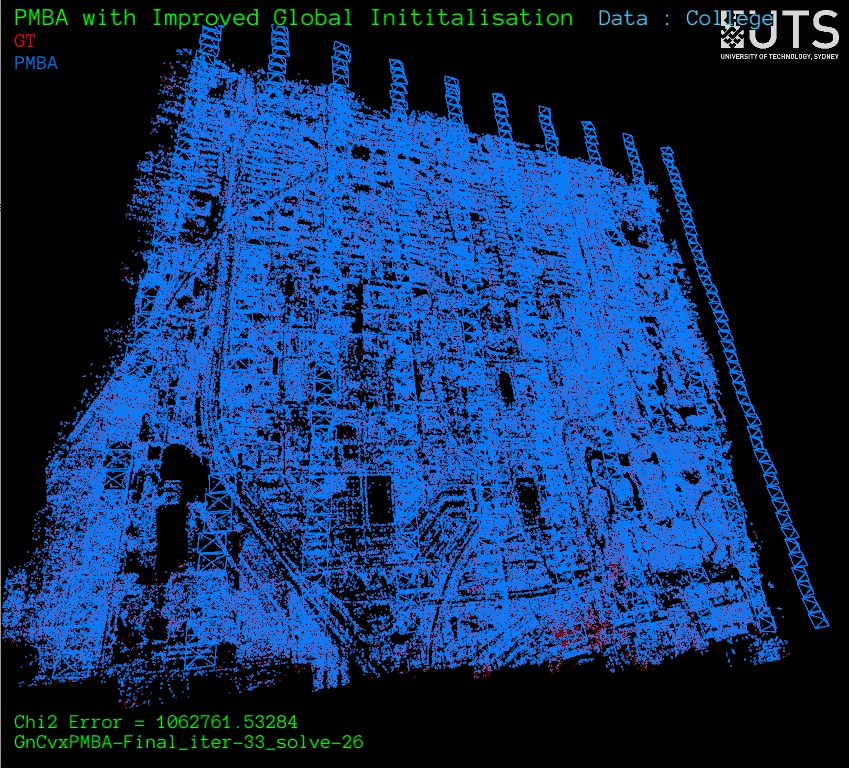} \\    
\end{tabular}	
  	\caption{
  	A demonstration of initialization to full PMBA pipeline: column 1 is QPLC results, column 2 shows selected iteration results in Convex initialization, column 3 is a typical iteration result in full-PMBA and column 4 shows the final map. Row 1: BAL-Ladybug-1370; row 2: BAL-Venice-427, row 3: BAL-Trafalgar-126, row 4: College aerial dataset}
  	\label{fig::qpcvxpmba-pipeline}	
\end{figure}

\section{Conclusion}
\label{Section::Conclusion}

In this work, we proposed a new bundle adjustment formulation (PMBA) which utilizes parallax angle based feature parametrization on manifold and observation-ray based objective function. We proved that under the new formulation the ill-conditioned cases due to problematic features can be theoretically avoided without any manual intervention, which results in much better convergence and robustness properties.
 
Furthermore, motivated by the strong local observability hidden in the visual SLAM problem, we derived a novel global initialization process for PMBA. We use a simplified model that can guarantee a near-optimal solution to bootstrap the original problem. Experimental results show that the proposed initialization can provide accurate estimates and is a viable global initialization strategy for many challenging situations including sequential and out-of-order images.
 
The promising results of the global initialization plus PMBA pipeline using publicly available datasets demonstrate that the proposed technique can deal with different challenging data. In the future, we are planning to integrate the proposed pipeline with efficient visual SLAM front-end to develop a robust and efficient SfM system.


\appendix

\subsection{The Proof of Theorem \ref{theorem::HFF}}
\label{appendix::theorem1}

Consider the feature $j$ and the corresponding sub-block matrix $\mathbf{H}_{\mathbf{FF}_j}$ in $\mathbf{H}_{\mathbf{FF}}= blkdiag(\mathbf{H}_{\mathbf{FF}_1}, \cdots, \mathbf{H}_{\mathbf{FF}_n} )$. 
Denoting $\mathbf{J}_{i,j} = \frac{\partial \mathbf{e}_{i,j}}{\partial \mathbf{F}_j} $ for ($i\in \mathbb{T}_j$), we have
\begin{equation}
	\begin{aligned}
	\mathbf{H}_{\mathbf{FF}_j} \geq \mathbf{J}_{m_j,j}^\intercal \mathbf{J}_{m_j,j} + \mathbf{J}_{a_j,j}^\intercal \mathbf{J}_{a_j,j}.
	\end{aligned}
\end{equation}
On the one hand,
\begin{equation}
	\begin{aligned}
\mathbf{J}_{m_j,j}^\intercal \mathbf{J}_{m_j,j} &= \begin{bmatrix}
0 & \mathbf{0} \\
\mathbf{0} &  (S(\mathbf{n}_j )\mathbf{A}_{\mathbf{n}_j})^\intercal (S(\mathbf{n}_j )\mathbf{A}_{\mathbf{n}_j})  
\end{bmatrix} \\&  =  \begin{bmatrix}
0 & \mathbf{0} \\
\mathbf{0} & \mathbf{I}_2 
\end{bmatrix}.
	\end{aligned}
\end{equation}
Denoting $\mathbf{a}_j =  \mathbf{p}_{m_j}-\mathbf{p}_{a_j}  $  and $ \mathbf{n}_{jw} = \mathbf{R}_{m_j}\mathbf{n}_f $, we have
\begin{equation}
\begin{aligned}
	\mathbf{N}_{j,m_a} =& \cos \theta_j \| \mathbf{a}_j \times \mathbf{n}_{jw}   \| \mathbf{n}_{jw} \\
	&+ \sin \theta_j (\mathbf{a}_j - (\mathbf{a}_j \cdot \mathbf{n}_{jw} ) \mathbf{n}_{jw} ).
\end{aligned}
\end{equation}
Note that $\mathbf{n}_{jw} \perp  (\mathbf{a}_j - (\mathbf{a}_j \cdot \mathbf{n}_{jw} ) \mathbf{n}_{jw} ) $ and 
$ \| (\mathbf{a}_j - (\mathbf{a}_j \cdot \mathbf{n}_{jw} ) \mathbf{n}_{jw} ) \| = \| \mathbf{a}_j \times \mathbf{n}_{jw}  \|: = \sin \gamma  $, thus we have 
\begin{equation}
	\begin{aligned}
  \mathbf{J}_{a_j,j}^\intercal \mathbf{J}_{a_j,j} & =  \begin{bmatrix}
  1 & \mathbf{0}\\
 \mathbf{0} & \mathbf{0}  
  \end{bmatrix}. 
	\end{aligned}
\end{equation}
Therefore, $\mathbf{H}_{\mathbf{F} \mathbf{F}_j} \geq \mathbf{I}_3 $ and $\mathbf{H}_{\mathbf{FF}} \geq \mathbf{I} $.

\subsection{The Proof of Theorem 2}
\label{appendix::theorem2}

Consider the following function
\begin{equation}
 h_{\mathbf{V}} ( \mathbf{x} ) : = 	\|  \breve{\mathbf{x}}    - \mathbf{V}  \|^2,
\end{equation}
where $\mathbf{x} \in \mathbb{R}^3$,  $\mathbf{V} \in \mathbb{R}^3$ ($\|\mathbf{V}\| = 1$).
It is a fact that
\begin{equation}
 h_{\mathbf{V} }(  \mathbf{x} + \lambda \Delta\mathbf{x}  ) \leq \max \{   h_{\mathbf{V} }(  \mathbf{x} ) ,  h_{\mathbf{V} }(  \mathbf{x}+ \Delta\mathbf{x} )  \}
 \label{eq::lemma}
\end{equation}
 for any $\Delta\mathbf{x} \in \mathbb{R}^3$ and any $ \lambda \in (0,1) $.
 
 Considering the problem (\ref{eq::optimisation3}) and the linearity of $\bar{\mathbf{N}}_{j,i}$ w.r.t. $\mathbf{p}$, (\ref{eq::optimisation3}) can be rewritten as   
 \begin{equation}
 \begin{aligned}
\min_{ \{ \mathbf{p} \} }  h( \mathbf{p}, \bar{\mathbf{R}}, \bar{\mathbf{F}}  )    & =
\min_{ \{ \mathbf{p} \} }  \sum_{ i \in \mathbb{T}_j, j } \| \frac{ \bar{\mathbf{A}}_i \mathbf{p} }{\| \bar{\mathbf{A}}_i \mathbf{p}  \|} -  \bar{\mathbf{V}}_{ij}  \|^2. \\
& = \min_{ \{ \mathbf{p} \} }  \sum_{ i \in \mathbb{T}_j, j } h_{\mathbf{V}_i}(  \bar{\mathbf{A}}_i \mathbf{p} ),
 \end{aligned}
\end{equation}
where $\bar{\mathbf{V}}_{ij}:=\bar{\mathbf{R}}_i \mathbf{v}_{ij}^{(l)}$ is a directional vector.

Denoting the global minimum of the  problem  (\ref{eq::optimisation3}) as $\bar{\mathbf{p}}$, we have 
\begin{equation}
	\begin{aligned}
&h( \bar{\mathbf{p}}+\lambda \Delta \mathbf{p}, \bar{\mathbf{R}}, \bar{\mathbf{F}}  )\\
 =& 
 \sum_{ i \in \mathbb{T}_j, j } h_{\mathbf{V}_i}(  \bar{\mathbf{A}}_i (  \bar{\mathbf{p}}+\lambda \Delta \mathbf{p}  ) )  \\
\small{(using (\ref{eq::lemma}) )} \leq &    
 \sum_{ i \in \mathbb{T}_j, j }   \max\{ h_{\mathbf{V}_i}(  \bar{\mathbf{A}}_i \bar{\mathbf{p}} ), h_{\mathbf{V}_i}(  \bar{\mathbf{A}}_i (\bar{\mathbf{p}} + \Delta \mathbf{p} ) )  \}         \\ 
 \leq &  h( \bar{\mathbf{p}} , \bar{\mathbf{R}}, \bar{\mathbf{F}}  ) +
 h( \bar{\mathbf{p}} + \Delta \mathbf{p}, \bar{\mathbf{R}}, \bar{\mathbf{F}}  )
	\end{aligned}  
\end{equation}
for any $ \Delta \mathbf{p} \in \mathbb{R}^{3M} $ and $\lambda \in (0,1)$.
The  inequality  above indicates a fact:
if we perform optimization for the problem (\ref{eq::optimisation3}), i.e., $\min_{\mathbf{p}} h(\mathbf{p}, \bar{\mathbf{R}},\bar{\mathbf{F}}  ) $
from  any initial guess $\mathbf{p} \in \mathbb{R}^{3M}$, the converged value $\mathbf{p}_{op}$ after optimization will be a near-optimal solution, i.e., 
$ h(\mathbf{p}, \bar{\mathbf{R}},\bar{\mathbf{F}}  ) \leq 2 h(\bar{\mathbf{p}}, \bar{\mathbf{R}},\bar{\mathbf{F}}  )$.
When $\bar{\mathbf{R}}$ is close to the optimal estimate,  $\mathbf{p}_{op}$ will be also a near-optimal solution of the problem (\ref{eq::optimisation2}) clearly. Under noise-free condition, $\mathbf{p}_{op}$ is an exact solution due to 
$ 0 \leq  h(\mathbf{p}, \bar{\mathbf{R}},\bar{\mathbf{F}}  ) \leq 2 h(\bar{\mathbf{p}}, \bar{\mathbf{R}},\bar{\mathbf{F}}  ) \leq 0 $ thus the  problem  (\ref{eq::optimisation3}) is convex.

\bibliographystyle{IEEEtran}
\bibliography{newRef.bib}


\end{document}